%% file: main.tex
\documentclass{article}

\PassOptionsToPackage{numbers,compress}{natbib}
\usepackage[main, final]{neurips_2026}

\usepackage[utf8]{inputenc} 
\usepackage[T1]{fontenc}    
\usepackage{url}            
\usepackage{booktabs}       
\usepackage{amsfonts}       
\usepackage{nicefrac}       
\usepackage{microtype}      
\usepackage{xcolor}         
\usepackage{amsmath}
\usepackage{amssymb}
\usepackage{mathtools}
\usepackage{amsthm}
\usepackage{url}
\usepackage{amsmath, amssymb}
\usepackage{graphicx}
\usepackage{booktabs}
\usepackage{amsthm}
\usepackage{listings}
\usepackage[T1]{fontenc}
\usepackage{algorithm}
\makeatletter
\@ifpackageloaded{algorithmic}{%

}{}
\makeatother
\usepackage{algorithmicx}

\usepackage{algpseudocode}
\usepackage{float}

\usepackage{enumitem}
\usepackage{xcolor}
\newcommand{\method}{MosaCD}
\newcommand{\sepsets}{\Sigma_{\min}}
\newcommand{\bk}{\mathcal{K}}
\usepackage{arydshln}
\usepackage{hyperref}       
\theoremstyle{plain}
\newtheorem{theorem}{Theorem}[section]

\newtheorem{lemma}[theorem]{Lemma}
\newtheorem{corollary}[theorem]{Corollary}
\theoremstyle{definition}

\newtheorem{assumption}[theorem]{Assumption}
\theoremstyle{remark}
\newtheorem{remark}[theorem]{Remark}

\title{Improving constraint-based discovery with robust propagation and LLM priors}

\author{%
  Ruiqi Lyu\thanks{Equal first-author contribution.}\\
  Carnegie Mellon University
  \And
  Alistair Turcan\footnotemark[1]\\
  Carnegie Mellon University
  \AND
  Martin Jinye Zhang\thanks{Equal last-author contribution.}\\
  Carnegie Mellon University
  \And
  Bryan Wilder\footnotemark[2]\\
  Carnegie Mellon University
}

\begin{document}

\maketitle

\begin{abstract}
Constraint-based causal discovery recovers causal DAG structure from conditional independence (CI) relations.
Classical methods such as PC orient v-structures first, then propagate edge directions from these seeds, relying on accurate CI tests and rich separating-set searches.
In practice, these conditions often fail, causing cascading orientation errors.
Recent work uses large language models (LLMs) as experts to augment edge orientation when standard assumptions fail, but often treats LLM outputs as reliable or assumes stable error behavior, despite hallucinations and instability.
We propose \method{}, a constraint-based framework that robustly combines CI tests with LLM inference to obtain high-confidence orientation seeds, then propagates them with Seeded Propagation Rules (SPR), which mitigate the fragility of collider-first orientation.
We prove oracle-level soundness for SPR under explicit assumptions on the skeleton, separating-set record, and seed orientations, and give a stylized finite-sample analysis showing why prioritizing non-collider evidence can reduce orientation errors.
Across 13 real-world benchmarks, \method{} and SPR achieve substantially higher accuracy than existing methods through more reliable seeds and more robust propagation.
\end{abstract}

\section{Introduction}

Causal discovery aims to recover cause-effect graphs over variables.
A prominent family, originating with PC \citep{spirtes2000causation}, is constraint-based discovery: CI tests first remove edges, then orientation rules direct the remaining ones.
Orientation starts from seed edges induced by v-structures and iteratively propagates directions.
These methods are widely used for their guarantees, flexibility across data types, and interpretable outputs \citep{spirtes2000causation}.
However, errors can accumulate at every step.
Theory assumes CI tests perfectly distinguish dependence from independence across conditioning subsets \citep{spirtes2000causation,kalisch2007estimating}, but finite-sample CI tests are noisy and exhaustive subset search is infeasible.
Orientation can amplify these errors: seed directions depend on fine-grained conclusions from separating sets found during skeleton search.
For an unshielded triple $X-Z-Y$, collider orientation requires the recorded minimal separating sets for $(X,Y)$ to consistently exclude $Z$; if this record is incomplete or corrupted by finite-sample CI errors, a false collider can be introduced and propagated.

Recent advances in LLMs create new opportunities for seeding edge orientations.
Despite their imperfections, LLMs encode broad knowledge useful for inferring pairwise causal relationships \citep{kiciman2023causal, vashishtha2025causal}.
Prior work prompts LLMs with causal queries, e.g., ``Does $A$ cause $B$?'' \citep{kiciman2023causal, vashishtha2023causal} or ``Is $A$ conditionally independent of $B$ given $C, D, \ldots$?'' \citep{cohrs2024large}, or uses LLM outputs as causal-order priors \citep{vashishtha2025causal} or enforced constraints \citep{hasan2023optimizing}.
However, most LLM-based causal discovery keeps the LLM and algorithm separate: LLM outputs are fed into existing algorithms as priors \citep{hasan2023optimizing}, or standard algorithms like PC are run first and the LLM is used for post hoc orientation \citep{vashishtha2025causal,khatibi2024alcm}.
In both cases, the discovery algorithm remains unchanged.

In this work, we study how causal discovery algorithms can directly integrate external domain knowledge.
We focus on two challenges: (i) generating useful orientation information without human experts, and (ii) propagating it to recover causal relations in the DAG.
We propose \method{}, a constraint-based causal discovery algorithm that combines CI tests with LLM-based inference to obtain high-quality orientation seeds, then propagates them using Seeded Propagation Rules (SPR), a new set of orientation rules that leverage reliable initial orientations unavailable to classical PC-style algorithms.
For robust LLM seeding, we use randomized prompting to reduce variance and mitigate hallucinations.
Theoretically and empirically, we show that with high-quality seeds, SPR's non-collider-first strategy yields more robust graph recovery than classical collider-first constraint-based methods.
Across 13 causal discovery benchmarks, \method{} achieves powerful constraint-based performance through more reliable seeding and propagation.

Our contributions are:
\begin{enumerate}
    \item We develop \method{}, a constraint-based causal discovery algorithm that robustly integrates domain knowledge by combining a reliable LLM-based seeding strategy with a tailored non-collider-first orientation strategy, termed Seeded Propagation Rules (SPR).
    \item We provide oracle-level soundness guarantees for SPR under explicit assumptions, and a stylized analysis showing why non-collider-first propagation is less exposed to finite-sample CI errors than collider-first propagation.
    \item We conduct extensive empirical evaluations across 13 real-world benchmark datasets spanning diverse settings, demonstrating consistent performance improvements from both the LLM-based seeding step and the proposed propagation strategy.
\end{enumerate}

\section{Related Work}

\textbf{Constraint-based causal discovery.}
Constraint-based causal discovery infers edges and orientations from CI relations and logical rules.
PC \citep{spirtes2000causation} removes edges via CI tests and orients the remainder with logical rules, assuming no hidden confounders.
FCI \citep{spirtes2013causal} extends PC to latent confounders and selection bias, returning PAGs (partial ancestral graphs) rather than CPDAGs.
PC-stable \citep{colombo2014order} removes PC's order dependence by fixing adjacency sets at each conditioning-set size.
Conservative PC (CPC) \citep{ramsey2012adjacency} reduces false colliders by requiring unanimity across observed separating sets before orienting an edge, while still assuming those sets are comprehensive.
Post-hoc consistency methods revisit CI tests to reconcile a PDAG with skeleton evidence.
PC-max \citep{ramsey2016improving} avoids contradictions by emphasizing conditioning sets with the most significant p-values.
These methods still rely on v-structures to seed orientations, then propagate the rest of the graph by avoiding new v-structures.

\textbf{Incorporating domain knowledge into constraint-based causal discovery.}
Constraint-based methods can naturally incorporate domain constraints or priors.
Tiered orders and path constraints can prune orientations after PC or its variants, then be closed under orientation rules \citep{meek2013causal}.
However, this still inherits the v-structure-first bias because domain knowledge is applied after initial collider orientations.
Beyond this paradigm, \citet{hyttinen2014constraint} encode tested (in)dependences and prior knowledge as logical constraints and minimize the total weight of violations, but the problem is NP-hard in the worst case and degrades on dense graphs.
\citet{claassen2012bayesian} assign Bayesian reliabilities to (in)dependence claims and process them in decreasing reliability, returning one model with confidence tags per decision, but require enumerating parts of the Markov equivalence class, which is computationally expensive.
Other causal discovery paradigms integrate domain knowledge differently, such as through a global optimization objective.

\textbf{LLMs for causal discovery.}
Recent work shows that LLMs can provide useful domain knowledge for causal discovery, but it remains difficult to distinguish accurate causal information from plausible errors.
Some methods directly query LLMs for causal constraints or full graphs, e.g., asking ``Does $A$ cause $B$?'' \citep{kiciman2023causal, vashishtha2023causal, jiralerspong2024efficientcausalgraphdiscovery} or ``Is $A$ conditionally independent of $B$ given $\{C, D, \ldots\}$?'' \citep{cohrs2024large}.
Others use LLM outputs as prior selectors \citep{vashishtha2025causal, havrilla2025igda, long2307causal, ban2025integrating} or discovery constraints \citep{takayama2024integrating}.
While effective in some settings, many treat direct LLM outputs as reliable or assume stable error behavior, without explicitly addressing hallucination and positional bias.
In contrast, \method{} uses customized querying and aggregation to identify and filter unreliable LLM responses before causal discovery.
Appendix~\ref{app:assumptions} discusses related LLM methods and assumptions in detail.

\textbf{Hallucination detection in LLMs.}
LLMs often fail to acknowledge uncertainty: when unsure on multiple-choice questions, they may select the first option rather than ``I don't know.''
Empirical studies report positional and presentation biases and reluctance to admit uncertainty; shuffle-and-vote mitigations help but do not eliminate the need for calibration \citep{wang2023large,pezeshkpour2023large}.
Hallucination-detection methods often require direct model access \citep{farquhar2024detecting} or fine-tuning \citep{zhang2024rtuninginstructinglargelanguage}.
Cheaper prompt-only methods typically rely on self-reflection \citep{manakul2023selfcheckgptzeroresourceblackboxhallucination} or LLM-generated confidence scores \citep{zhao2024factandreflectionfarimprovesconfidence}, requiring the model to reason about its own errors.
Causal discovery methods that extract LLM priors would benefit from scalable, prompt-only hallucination filtering that requires no calibration.

\section{Preliminaries}

\textbf{Notations.} 
Let $G=(V,E)$ be the (unknown) ground-truth DAG, where $V$ is the set of observed variables and $E$ is the edge set. 
For disjoint $X,Y,S \subseteq V$, write $X \perp Y \mid S$ if $S$ $d$-separates $X$ and $Y$ in $G$, and $X \not\perp Y \mid S$ otherwise. 
We use $X-Y$ for an undirected edge, $X \to Y$ for a directed edge, and $X\to\cdots\to Y$ for a directed path.

We use curly braces $\{X,Y,\ldots\}$ to denote an unordered node set.
A triple $X-Z-Y$ is \emph{unshielded} if $X$ and $Y$ are non-adjacent but both are adjacent to $Z$. 
A partially directed graph (PDAG) is a graph whose edges may be directed or undirected; throughout, an acyclic PDAG means one with no directed cycle.

\textbf{PC-style skeleton search (PC, CPC, PC-stable).} 
The procedure \citep{spirtes2000causation} starts from the complete undirected graph $K_V$ and removes an unordered edge between $X$ and $Y$ whenever a conditional independence (CI) test accepts $X \perp Y \mid S$ for some $S \subseteq V \setminus \{X,Y\}$. 
The resulting undirected graph is $\widehat{Skel}=(V,\hat E)$.
For each removed edge $\{X,Y\}\notin \hat E$, we maintain a recorded minimal separating-set collection $\sepsets(X,Y)=\sepsets(Y,X)\subseteq 2^{V\setminus\{X,Y\}}$. Operationally, $\sepsets(X,Y)$ contains accepted separating sets of minimum cardinality among the sets tested for $(X,Y)$; equivalently, accepted supersets of already accepted separating sets are not used for orientation. Each recorded set $S$ is stored with its CI-test p-value $p_{X,Y,S}$; when a score is needed, write $p_{\max}(X,Y)=\max_{S\in\sepsets(X,Y)}p_{X,Y,S}$. PC and PC-stable typically retain one such set per removed edge, whereas CPC records multiple sets at the first separating level.

\textbf{Intuition of \method{} and SPR.}
Given a skeleton, PC-style algorithms need seed orientations before further propagation, e.g., via Meek's rules \citep{meek2013causal}.
They obtain these seeds from unshielded colliders: for each unshielded triple $X-Z-Y$, if $Z\notin S$ for every $S\in\sepsets(X,Y)$, they orient $X\to Z\leftarrow Y$.
Conversely, if $Z\in S$ for every $S\in\sepsets(X,Y)$, the recorded minimal separators support treating $Z$ as a non-collider; if both separator types appear, conservative procedures leave the triple unresolved.
With an oracle and complete minimal-separator records, collider and non-collider decisions follow the same separator criterion.
In finite samples, however, CI tests can falsely accept independence, and level-wise skeleton search tests many sets excluding $Z$.
A false accepted separator excluding $Z$ can create spurious collider evidence, which is especially harmful as collider orientations seed propagation.
SPR instead propagates first from external seed orientations and prioritizes non-collider evidence from $\sepsets(X,Y)$ before applying collider orientations.
This lets external domain knowledge resolve directions in non-collider configurations that PC alone cannot distinguish.

\section{Method}

\method{} takes as input a dataset $\mathcal{D}$, variables $V$, and known edge orientations $\mathcal{O}$, by default generated by an LLM reasoning over variables and CI tests.
Its core SPR stage outputs a PDAG; an optional vote-completion step orients remaining reversible edges into a DAG for evaluation.
The oracle guarantee in Section~\ref{subsec_main:correctness} concerns only the SPR PDAG, not this optional completion. Stage-wise results appear in Figure~\ref{fig:rebuttal:stages} (Appendix~\ref{app:rebuttal:stages}).
\method{} has four steps.
First, it builds an undirected skeleton using a constraint-based method, e.g., PC, CPC, or PC-stable, producing $G_{\text{skel}}$ and recorded minimal separating sets $\sepsets$ with CI p-values.
Second, it applies domain knowledge $\mathcal{O}$ to seed skeleton edges; $\mathcal{O}$ can be LLM-generated from variable names, descriptions, and $\sepsets(X,Y)$, with shuffled answer order and repeated queries to reduce positional bias and hallucination.
Third, SPR propagates orientations across the seeded graph.
For unshielded triples $X-Z-Y$, SPR uses $\sepsets(X,Y)$ to prioritize \emph{non-collider} evidence, where $Z$ appears in every recorded minimal separator, before collider evidence, where $Z$ appears in none.
SPR then resolves some remaining undirected edges by choosing the orientation with the fewest conflicts with the recorded CI evidence $\sepsets$, leaving ties undirected.
Finally, optional vote completion orients leftover undirected edges using a topological order from Step 2 LLM votes.

\textbf{Step 1: Skeleton search.} 
\method{} initializes a skeleton $G_{\text{skel}}$ and a recorded minimal separating-set collection $\sepsets$ from dataset $\mathcal{D}$ using a constraint-based algorithm such as PC, CPC, or PC-stable. 
The PDAG $P$ is initialized as $G_{\text{skel}}$. 
For each removed edge $\{X,Y\}$, $\sepsets(X,Y)$ contains one or more minimum-cardinality accepted separating sets, together with the corresponding CI p-values. PC and PC-stable usually contribute a single recorded separator, while CPC may contribute several. All SPR orientation decisions below use $\sepsets(X,Y)$.

\textbf{Step 2: Domain knowledge orientation seeding.} 
Given known orientations $\mathcal{O}$ in the skeleton, orient all such edges in the skeleton following this knowledge.

\begin{enumerate}
    \renewcommand{\labelenumi}{2.\arabic{enumi}}
    \item \textbf{\textbf{(Default) Generate domain knowledge from LLMs}} 
If seed orientations $\mathcal{O}$ are unavailable, we generate high-confidence seeds for undirected edges by querying an LLM.
For each undirected edge $X-Y$, we provide variable names, descriptions, and $\sepsets(X,Y)$ with CI p-values.
To reduce positional bias and hallucination, we randomize candidate answer order, e.g., ``$X\to Y$ or $Y\to X$'' and its reverse, and repeat each query five times.
We retain edges with consistent answers, defined as the majority vote under both answer orders.
We discard any proposed seed $X\to Y$ that (i) contradicts separator evidence in $\sepsets$ at any unshielded triple or (ii) creates a directed cycle in $P$.
\end{enumerate}

\textbf{Step 3: Apply SPR to propagate edges from seeds.}
\begin{enumerate}
    \renewcommand{\labelenumi}{\arabic{enumi}.}

    \item \textbf{Iterative orientation propagation.}
    SPR first rejects the candidate if it would create a directed cycle, then repeats the following steps until convergence. 
    \begin{enumerate}
        \renewcommand{\labelenumii}{1.\arabic{enumii}}
        \item \textbf{Unsupervised acyclic propagation:}
        Before CI evidence is certified, apply the directed-path acyclicity rule: if $X-Y$ and $X\to\cdots\to Y$ in $P$, then orient $X\to Y$. 
        \item \textbf{CI-supervised propagation:}
        For each unshielded and partially ordered triple $X \to Z - Y$, sort in descending order by $\max p$ in $\sepsets(X,Y)$, since larger p-values provide stronger evidence for conditional independence. Orient $Z \to Y$ if $Z$ appears in every recorded minimal separator in $\sepsets(X,Y)$, and orient $Y \to Z$ if $Z$ appears in none.
        Prioritizing by $\max p$ ensures that orientations with stronger CI support are applied first.
        \item \textbf{Collider orientation:}
        For each unshielded and unordered triple $X - Z - Y$, again sort by descending $\max p$ in $\sepsets(X,Y)$.
        Orient $X \to Z \leftarrow Y$ if $Z$ appears in no recorded minimal separator in $\sepsets(X,Y)$.
        \item \textbf{Certified Meek closure:}
        Apply a Meek rule only when every unshielded collider/non-collider fact needed by its premise has been certified by $\sepsets$. Reject orientations that would create a directed cycle.
    \end{enumerate}

    \item \textbf{Least-conflict orientation.}
   SPR can optionally resolve remaining undirected edges by choosing the direction least conflicting with recorded conditional independences in $\sepsets$.
    For each undirected pair $X-Y$ in random order, we consider $X\to Y$ and $Y\to X$.
    For each candidate, we close the graph under the orientation rules and count local separator-triple contradictions in $\sepsets$; candidates creating a directed cycle are rejected.
    We choose the lower-conflict direction and leave $X-Y$ undirected on ties.
    This finite-sample step is inactive in the oracle setting of Theorem~\ref{thrm:correctness}: rule closure has already exhausted all sound compelled orientations, so least-conflict adds none. 
\emph{Example.} After SPR Step 1, suppose $U\to Y$, $V\to Y$, $W\to Y$, $X-Y$ is undirected, and the recorded CI evidence includes $\{Y\}\in\sepsets(X,U)$, $\{Y\}\in\sepsets(X,V)$, and $\emptyset\in\sepsets(X,W)$. Then $X\to Y$ creates the colliders $X\to Y\leftarrow U$ and $X\to Y\leftarrow V$ (2 conflicts), whereas $Y\to X$ makes $W\to Y\to X$ induce $X\not\perp W$ despite $\emptyset\in\sepsets(X,W)$ (1 conflict); by the least-conflict step, we choose $Y\to X$.
\end{enumerate}

\textbf{(Optional) Step 4: Final orientation via votes.}
If edges remain undirected, we use the Step 2 LLM votes to complete the DAG, analogous to \citet{vashishtha2025causal}, but with shuffled answer orders to reduce positional bias.
We aggregate votes into a weighted directed graph and remove the weakest edges, defined by net directional support, to break cycles.
We then derive a topological order and orient any remaining undirected edges accordingly, producing a final DAG for evaluation.
This optional step is outside the oracle CPDAG/PDAG guarantee in Section~\ref{subsec_main:correctness} and is used only when a fully oriented DAG is required.

\section{Theoretical Analysis}
We start by analyzing the oracle behavior of SPR. The goal of this result is narrower than full finite-sample correctness: it shows that, when the skeleton, recorded minimal separators, and seed orientations are correct and mutually compatible, SPR's rule-closure stage does not introduce orientations that violate these constraints. We distinguish the ordinary CPDAG from the PDAG obtained after adding background orientation constraints, since external seed orientations may select a subset of the Markov equivalence class. We assume causal sufficiency, Markovness, faithfulness, a perfect CI oracle, and a seed oracle that returns only orientations compatible with the target background-knowledge class, though it may abstain. 
We then examine departures from these assumptions, particularly the noisiness in CI tests that motivates SPR.
To complement our empirical results, we provide a theoretical analysis in a stylized model, demonstrating that orienting non-colliders first (as SPR does) yields fewer errors than orienting colliders first (as in PC algorithms).       

\subsection{Correctness \label{subsec_main:correctness}}
We distinguish two objects. First, the CPDAG of a DAG $G$ is the unique PDAG representing the Markov equivalence class $[G]$: it has the same skeleton and unshielded colliders as $G$, directed edges are those compelled in every DAG in $[G]$, and undirected edges are reversible within $[G]$ \citep[Theorem 4.1]{andersson1997characterization}. Second, when background orientation constraints are provided, let $\bk$ denote the set of accepted seed orientations and define $[G]_{\bk}=\{H\in [G]: H$ satisfies every orientation in $\bk\}$.
When $[G]_{\bk}\neq \emptyset$, the $\bk$-compatible essential graph is the maximally oriented PDAG representing this restricted class: an edge is directed iff it has the same orientation in every DAG in $[G]_{\bk}$, and is left undirected otherwise. For $\bk=\emptyset$, this reduces to the ordinary CPDAG.

\begin{theorem}\label{thrm:correctness}
    Let $G$ be a causally sufficient DAG over $V$, and let $\bk$ be a set of background orientation constraints on edges in the true skeleton. Assume:
    (i) \textbf{Markov and faithfulness:} the observational distribution is Markov and faithful to $G$;
    (ii) \textbf{Perfect CI oracle and skeleton consistency:} skeleton search returns $\widehat{Skel}=\mathrm{Skel}(G)$ and the recorded minimal separator collection $\sepsets$ is correct;
    (iii) \textbf{Compatible seeds:} the seed set $E_{\mathrm{seed}}=\bk$ is acyclic and compatible with at least one DAG in $[G]$, i.e., $[G]_{\bk}\neq\emptyset$;
    (iv) \textbf{Sound propagation rules:} SPR Step~1 applies only acyclicity-preserving orientation rules and CI-guarded collider/non-collider implications derived from $\sepsets$.
    Then every orientation made by SPR Step~1 is sound with respect to the restricted equivalence class $[G]_{\bk}$: after each orientation, at least one DAG in $[G]_{\bk}$ remains consistent with the current PDAG, the recorded minimal separators, and the seed constraints. The resulting PDAG is an acyclic PDAG with skeleton $\mathrm{Skel}(G)$ satisfying the recorded separator constraints and the seed constraints $\bk$. If $\bk=\emptyset$, then under the same oracle and faithfulness assumptions, SPR Step~1 reduces to the usual PC-style collider orientation followed by orientation-rule closure, and returns the ordinary CPDAG of $G$; in this oracle-faithful setting its closure coincides with PC, PC-stable, and CPC.
\end{theorem}

See Appendix \ref{appendix:correctness} for the proof. 

\begin{remark}
    Theorem~\ref{thrm:correctness} is an oracle-level statement about the SPR rule-closure stage. It does not assume that LLM-provided seeds are always correct in finite samples, and it does not cover the optional vote-completion step that orients the remaining reversible edges into a DAG. If retained seeds are incorrect, exact recovery of the true CPDAG cannot be guaranteed in general; our finite-sample robustness to such errors is evaluated empirically in Figure~\ref{fig:secondary} and Appendix~\ref{app:propagation}.
\end{remark}

\subsection{Prioritizing non-collider over collider identification improves accuracy\label{subsec_main:accuracy}}
For an unshielded triple $X-Z-Y$, PC-style collider orientation checks whether $Z$ is absent from the recorded minimal separating sets $\sepsets(X,Y)$.
SPR instead first uses non-collider evidence, namely cases where $Z$ appears in every recorded minimal separator in $\sepsets(X,Y)$, before applying collider orientations. With a complete and correct minimal-separator record, collider and non-collider status are both determined by $\sepsets(X,Y)$. However, under noisy CI tests with finite samples, collider-first procedures have more opportunities to accept a spurious separator that excludes $Z$, thereby creating an erroneous early collider seed. We analyze this effect in a stylized level-wise search model.

\paragraph{Level-wise search and error events.}
Let $\ell = |C|$ denote the conditioning-set size. 
The search proceeds by levels $\ell = 0,1,2,\dots$, testing $X \perp Y \mid C$ over all $C \subseteq V \setminus \{X,Y\}$ with $|C|=\ell$. 
At level $\ell$, if any candidate $C$ is accepted as a sepset, the search stops; we declare $Z$ a collider if $Z \notin C$, and a non-collider if $Z \in C$.
Accordingly, a collider error occurs if $Z$ is a non-collider but $Z \notin C$, and a non-collider error occurs if $Z$ is a collider but $Z \in C$. 
We measure their relative frequency via
\begin{align}
    \mathcal{R}_\ell := \frac{\Pr(\text{collider error at level }\ell)}{\Pr(\text{non-collider error at level }\ell)}.
\end{align}
At level $\ell$, PC uses the \textbf{first} accepted sepset $C$ for $(X,Y)$. PC-stable has the same collider/non-collider decision as PC, but with adjacency sets frozen within level $\ell$ (order-invariant), so $\mathcal{R}_\ell$ matches PC. At the first level $\ell$ where independence holds, CPC gathers all minimal sepsets for $(X,Y)$ and orient the collider iff $Z$ is in \textbf{none}, treat as non-collider iff $Z$ is in \textbf{all}, otherwise leave the triple unoriented.

We compute $\mathcal{R}_\ell$ for the PC (the same as that of PC-stable) and CPC rules (denoted $\mathcal{R}_\ell^{\mathrm{PC}}$ and $\mathcal{R}_\ell^{\mathrm{CPC}}$).
PC-stable has the same $\mathcal{R}_\ell$ as PC, since it only removes within-level order dependence. Whenever $\mathcal{R}_\ell > 1$ for a given rule set, tests of non-colliders (prioritized by SPR) will have a lower error rate than test of colliders (prioritized by existing algorithms).
We make the following assumptions to make theoretical analysis tractable, noting they are not required for running SPR in practice.

\begin{assumption} \label{assumption:SimpleCI}
    (Simple CI test model) 
    Conditional independence (CI) tests act independently across candidates/levels given truth labels, with false positive rate $\alpha$ and false negative rate $\beta$ that do not vary with $\ell$ or $C$. Thus, a true sepset rejects dependence with probability $1-\beta$, while a non-sepset does so with probability $\alpha$. 
\end{assumption}

\begin{assumption}($Z$ controls the $X-Y$ path) \label{assumption:purity}
    (a) All $X-Y$ paths of length at most $2\ell+1$ pass through $Z$; (b) whether the $X-Y$ path is open is determined fully by whether $Z$ is conditioned on. 
\end{assumption}

Similar to related work \citep{ramsey2012adjacency, vashishtha2025causal}, these assumptions are deliberately simplified to make analysis tractable. Here, they let us illustrate the core dynamic that non-collider evidence is less error-prone by isolating the impact of the single node $Z$ (intuitively, Assumption \ref{assumption:purity} describes a locally sparse graph without redundant $X-Y$ paths) and imposing a single set of parameters describing the performance of the CI test. They should not be read as conditions required for running \method{} or as a complete model of finite-sample CI testing. Assumption~\ref{assumption:SimpleCI} replaces heterogeneous and correlated CI-test behavior with a homogeneous independent-error model in order to isolate the orientation-order effect. Assumption~\ref{assumption:purity} is a local structural simplification rather than a claim that all real graphs have a single $X-Y$ path. Across the 13 datasets, the local condition holds for 71\% of relevant triples on average. Performance remains strong even when the condition holds less often; for example, it holds for 42\% of relevant triples in Child, where \method{} still achieves the best F1 score (Table \ref{tab:comparison}).

In this model, we obtain exact analytical expressions for $R_\ell^{\mathrm{PC}}$ and $R_\ell^{\mathrm{CPC}}$, derived and shown in the appendix. These expressions involve a number of combinatorial quantities, but in the asymptotic regime where the error rates of the CI test are small relative to the graph size (a necessary condition for the algorithm to not be overwhelmed with errors), we can further simplify and show that $R_\ell^{\mathrm{PC}}$ and $R_\ell^{\mathrm{CPC}}$ must be strictly above 1.

\begin{theorem} \label{thrm:odds}
    Let $M = |V\setminus\{X,Y\}|$. Suppose that $\alpha, \beta = o\left(\frac{1}{M}\right)$ and $\ell = \Theta(1)$. For $M$ sufficiently large compared to $\ell$, the error ratios satisfy
    \begin{align*}
    &R_\ell^{\mathrm{CPC}} = \beta^{\binom{M-1}{\ell-1}-\binom{M-1}{\ell}}\cdot\frac{M-\ell}{\ell} + o\left(\frac{1}{M}\right) > 1 \\
    &R_\ell^{\mathrm{PC}} = \frac{n(n+1)}{m(m+1)}(1+o(1)) = \left(\frac{M-\ell}{\ell}\right)^2(1+o(1)) > 1,
    \end{align*}
    where $m=\binom{M-1}{\ell-1}, n=\binom{M-1}{\ell}$. Thus, for both algorithms, the error rate among colliders will be higher than noncolliders.
\end{theorem}

See Appendix \ref{appendix:accuracy} for the proof. In this stylized noisy regime, the intuition is that there are more candidate subsets that do \textit{not} contain $Z$ than subsets that do contain $Z$, so collider-first strategies have more opportunities to make a mistake at a reached level. SPR flips the ordering, starting with candidates that do contain $Z$, since this set is smaller and the total number of mistakes made in early stages will be limited. Numerical experiments (Appendix \ref{subsec:numerical}) also show substantially lower FPRs under the non-collider first strategy across three standard skeleton learners, consistent with the theory.

\section{Experimental results}

\begin{table*}[htb!]
\footnotesize
\centering
\setlength{\tabcolsep}{4pt}
\begin{tabular}{llccccccc}
\hline
\multicolumn{2}{c}{} & \multicolumn{7}{c}{F1/SHD} \\
\hline
Group & Method & \shortstack{Cancer\\(5 nodes)} & \shortstack{Asia\\(8 nodes)} & \shortstack{Covid Resp.\\(11 nodes)} & \shortstack{ALZ\\(11 nodes)} & \shortstack{Child\\(20 nodes)} & \shortstack{Insurance\\(27 nodes)} & \shortstack{Water\\(32 nodes)} \\
\hline
Non-LLM & PC & \underline{0.50}/\underline{2} & \underline{0.67}/\underline{3} & 0.31/\underline{18} & 0.49/13 & 0.70/11 & 0.62/28 & 0.45/53 \\
 & Meek & \underline{0.50}/\underline{2} & \underline{0.67}/\underline{3} & 0.32/\textbf{16} & \underline{0.53}/\textbf{11} & \textbf{0.78}/\textbf{6} & \textbf{0.70}/\textbf{19} & \textbf{0.57}/\textbf{38} \\
 & Shapley-PC & \textbf{1.00}/\textbf{0} & 0.53/4 & 0.28/26 & 0.31/25 & 0.67/19 & \underline{0.67}/28 & \underline{0.47}/53 \\
\hdashline
 & NOTEARS & 0.00/4 & 0.57/5 & \underline{0.45}/\textbf{16} & 0.07/19 & 0.48/17 & 0.15/50 & 0.29/54 \\
 & DAG-GNN & 0.00/4 & 0.27/9 & \textbf{0.46}/25 & \textbf{0.68}/16 & 0.24/23 & 0.20/60 & 0.26/55 \\
 & GES & \underline{0.50}/\underline{2} & 0.12/9 & 0.33/\underline{18} & 0.42/25 & 0.17/63 & 0.26/105 & 0.29/71 \\
 & HC & \underline{0.50}/\underline{2} & 0.12/9 & 0.33/\underline{18} & \underline{0.53}/17 & 0.33/51 & 0.33/95 & 0.36/60 \\
 & MMHC & \textbf{1.00}/\textbf{0} & \textbf{0.93}/\textbf{1} & 0.32/19 & 0.07/19 & 0.60/12 & 0.61/23 & 0.29/71 \\
\hdashline
 & SPR (Ours) & \underline{0.50}/\underline{2} & \underline{0.67}/\underline{3} & 0.31/\underline{18} & 0.47/\underline{12} & \underline{0.76}/\underline{7} & \underline{0.67}/\underline{22} & 0.45/\underline{51} \\
Oracle & \shortstack{SPR + one\\oracle edge} & 0.60/1 & 0.71/2 & 0.31/18 & 0.52/11 & 0.79/7 & 0.74/17 & 0.51/46 \\
\hline
LLM & ILS-CSL$^*$ & 0.50/2 & \textbf{0.93}/\textbf{1} & 0.31/18 & 0.49/13 & \underline{0.83}/\underline{6} & \underline{0.70}/25 & 0.60/38 \\
 & SCP$^*$ & 0.50/2 & 0.67/3 & 0.39/15 & 0.53/11 & 0.78/\underline{6} & 0.68/\underline{21} & 0.57/38 \\
 & Jiralerspong* & 0.57/3 & 0.55/5 & 0.25/18 & 0.48/15 & 0.10/38 & 0.21/52 & 0.00/71 \\
 & Vashishtha* & 0.75/\underline{1} & \textbf{0.93}/\textbf{1} & \underline{0.65}/\underline{11} & 0.48/16 & 0.67/9 & \underline{0.70}/\underline{21} & \underline{0.65}/\underline{34} \\
 & Causal Disco* & 0.57/2 & \underline{0.86}/\underline{2} & \textbf{0.72}/\textbf{6} & \underline{0.68}/\underline{10} & 0.56/13 & 0.34/41 & \textbf{0.87}/\textbf{15} \\
 & LLM-CD* & \textbf{1.00}/\textbf{0} & 0.80/\underline{2} & 0.39/15 & 0.49/13 & 0.40/30 & 0.29/59 & 0.24/81 \\
 & Memorization* & \underline{0.80}/2 & 0.78/4 & 0.55/14 & 0.55/12 & 0.32/26 & 0.35/47 & 0.53/42 \\
\hdashline
 & MosaCD$^*$ (Ours) & \textbf{1.00}/\textbf{0} & \textbf{0.93}/\textbf{1} & \underline{0.65}/\underline{11} & \textbf{0.73}/\textbf{8} & \textbf{0.90}/\textbf{2} & \textbf{0.87}/\textbf{12} & 0.59/36 \\
\hline
\end{tabular}

\vspace{0.25em}

\begin{tabular}{llcccccc}
\hline
Group & Method & \shortstack{Mildew\\(35 nodes)} & \shortstack{Alarm\\(37 nodes)} & \shortstack{Hailfinder\\(56 nodes)} & \shortstack{Covid. Comp.\\(63 nodes)} & \shortstack{Hepar2\\(70 nodes)} & \shortstack{Win95pts\\(76 nodes)} \\
\hline
Non-LLM & PC & 0.63/23 & 0.85/12 & 0.38/105 & 0.23/152 & 0.36/100 & 0.59/68 \\
 & Meek & \underline{0.69}/\textbf{17} & \textbf{0.90}/\textbf{8} & \underline{0.40}/97 & \underline{0.28}/\textbf{134} & 0.39/\underline{87} & 0.64/\underline{58} \\
 & Shapley-PC & \textbf{0.75}/\underline{19} & 0.84/14 & 0.38/105 & 0.18/174 & 0.44/89 & \underline{0.65}/70 \\
\hdashline
 & NOTEARS & 0.25/44 & 0.25/40 & 0.27/\textbf{58} & 0.13/\underline{144} & 0.06/119 & 0.14/111 \\
 & DAG-GNN & 0.43/50 & 0.53/35 & 0.13/77 & 0.07/648 & 0.13/119 & 0.37/93 \\
 & GES & 0.23/152 & 0.26/101 & 0.33/150 & \textbf{0.29}/302 & 0.45/122 & 0.51/122 \\
 & HC & 0.24/140 & 0.30/92 & 0.31/162 & \textbf{0.29}/300 & \underline{0.47}/119 & 0.46/133 \\
 & MMHC & 0.52/27 & 0.68/16 & \textbf{0.59}/\underline{72} & 0.17/156 & \textbf{0.67}/\textbf{55} & \textbf{0.68}/\textbf{50} \\
\hdashline
 & SPR (Ours) & 0.64/21 & \underline{0.87}/\underline{10} & \underline{0.40}/100 & 0.22/151 & 0.37/91 & 0.61/63 \\
Oracle & \shortstack{SPR + one\\oracle edge} & 0.86/10 & 0.87/10 & 0.43/95 & 0.23/148 & 0.48/80 & 0.73/50 \\
\hline
LLM & ILS-CSL$^*$ & \underline{0.89}/\underline{10} & 0.85/12 & 0.44/97 & 0.00/138 & 0.54/78 & 0.69/54 \\
 & SCP$^*$ & 0.69/17 & \underline{0.87}/\underline{8} & 0.39/99 & \underline{0.27}/129 & 0.38/88 & 0.63/56 \\
 & Jiralerspong* & 0.07/50 & 0.09/80 & 0.10/\underline{72} & 0.03/145 & 0.16/137 & 0.03/135 \\
 & Vashishtha* & 0.78/13 & 0.72/15 & \textbf{0.60}/77 & \underline{0.27}/131 & 0.66/59 & \underline{0.74}/\underline{46} \\
 & Causal Disco* & 0.61/24 & 0.51/27 & 0.33/\textbf{53} & \textbf{0.39}/\underline{120} & \textbf{0.81}/\textbf{36} & 0.55/69 \\
 & LLM-CD* & 0.24/52 & 0.30/59 & 0.22/92 & \underline{0.27}/129 & 0.32/117 & 0.17/160 \\
 & Memorization* & 0.24/59 & 0.29/46 & 0.17/79 & 0.17/156 & 0.12/147 & 0.00/182 \\
\hdashline
 & MosaCD$^*$ (Ours) & \textbf{0.90}/\textbf{8} & \textbf{0.93}/\textbf{6} & \underline{0.49}/90 & \textbf{0.39}/\textbf{118} & \underline{0.72}/\underline{52} & \textbf{0.81}/\textbf{37} \\
\hline
\end{tabular}

\caption{\textbf{Benchmark evaluation.} Each entry reports F1/SHD for each dataset and method (nodes in brackets). Best F1 and SHD is \textbf{bold} and second-best is \underline{underlined} within the LLM and non-LLM groups. Oracle rows receive one ground-truth edge and are separate diagnostics, excluded from non-oracle rankings.}
\label{tab:comparison}
\end{table*}

We use 13 benchmark datasets: 10 from BNLearn \citep{Scutari2014}---Cancer, Asia, Child, Insurance, Water, Mildew, Alarm, Hailfinder, Hepar2, and Win95pts---and 3 from Srivastava et al.~\citep{srivastava2025realizing}: Covid Resp., ALZ, and Covid. Comp., which were released after the GPT-4o-mini training cutoff used in our experiments and were designed to reduce benchmark memorization.
To reduce memorization, we replace original variable names with human-readable versions less easily linked to the datasets, e.g., ``HrglssDrtnAftrPrnt'' becomes ``Hourglass Duration Too Long.''
The datasets span 5--76 nodes and include real and simulated graphs with categorical and continuous data.
For simulated datasets, we generate 20,000 samples each.
We evaluate true edge orientation recovery using F1 score and structural Hamming distance (SHD). Additional 211-node MACH and 283-node BEELINE evaluations are reported separately in Appendix~\ref{app:rebuttal:scale}.

\textbf{Baselines.}
We use three skeleton methods---PC \citep{spirtes2000causation}, PC-stable \citep{colombo2014order}, and CPC \citep{ramsey2012adjacency}---all compatible with our orientation strategies.
Given a skeleton, we evaluate 15 baselines:
(i) \textbf{PC}, using the corresponding skeleton method's standard orientation procedure;
(ii) \textbf{Meek} \citep{meek2013causal}, applying Meek's rules after PC;
(iii) \textbf{Shapley-PC} \citep{russo2023shapley}, using Shapley-value feature importance;
(iv) \textbf{ILS-CSL} \citep{ban2023causal}, an LLM method incorporating statistical knowledge;
(v) \textbf{SCP} \citep{cohrs2024large}, an LLM orientation method;
(vi) \textbf{Jiralerspong} \citep{jiralerspong2024efficientcausalgraphdiscovery}, a pairwise LLM method prompting on two nodes;
(vii) \textbf{Vashishtha} \citep{vashishtha2025causal}, a pairwise LLM causal-ordering method given the full graph;
(viii) \textbf{Causal Disco} \citep{long2307causal}, using LLM information as a noise-robust prior;
(ix) \textbf{LLM-CD} \citep{ban2025integrating}, also using noisy LLM priors;
(x--xiv) non-LLM methods, \textbf{NOTEARS} \citep{zheng2018dagstearscontinuousoptimization}, \textbf{DAG-GNN} \citep{yu2019dag}, \textbf{GES} \citep{chickering2002optimal}, \textbf{Hill climbing (HC)} \citep{heckerman1995learning}, and \textbf{MMHC} \citep{tsamardinos2006max}, implemented in gCastle \citep{zhang2021gcastle} and run independently of the skeleton, except MMHC, initialized with the same skeleton as \method{};
and (xv) \textbf{Memorization}, which prompts the LLM to return a causal graph from all variables.
These cover state-of-the-art LLM and non-LLM constraint-based, score-based, and continuous-optimization methods.
All methods access the dataset; LLM methods also receive identical variable names, metadata, and GPT-4o-mini backbone (Appendix~\ref{app:rebuttal:resources}).
We compare \textbf{\method{}}, the full pipeline, and \textbf{SPR}, propagation without domain knowledge. We separately report \textbf{SPR + one oracle edge}, propagation with one random ground-truth orientation averaged over 100 replicates, as an oracle diagnostic. Matched oracle-edge controls appear in Appendix~\ref{app:rebuttal:oracle}.
For fairness, LLM and non-LLM methods are compared only within groups.
See Appendix~\ref{app:prompt} for details.

\subsection{Benchmarking experiments}

We applied \method{}, SPR, and baselines to 13 benchmark datasets.
Table~\ref{tab:comparison} reports results using PC skeletons; PC-Stable and CPC results appear in Appendix~\ref{app:main}.
We draw two main conclusions.
First, \method{} broadly outperforms LLM baselines, achieving the best performance on 9 of 13 datasets and the best overall results: average F1 of 0.76 versus 0.66 for the next-best method, Vashishtha, and average SHD of 29 versus 32 for the next-best method, Causal Disco.
This includes the three datasets released after the LLM training cutoff.
Second, the oracle diagnostic SPR + one oracle edge improves over Meek on 7 datasets and ties on at least one metric on 2 more; this comparison includes extra ground-truth information, so matched-access results are reported separately in Appendix~\ref{app:rebuttal:oracle}.
Adding just one edge can sharply improve SPR by reducing reliance on collider orientations; for example, in Mildew, F1 increases by 0.22.
Consistent with our theory, PC skeleton orientations identify undirected edges with 92.2\% accuracy, while collider orientations are correct only 73.1\% of the time.
\method{} similarly outperforms baselines with PC-Stable and CPC skeletons (Appendix~\ref{app:main}).

\begin{figure}[htb!]
    \centering
    \includegraphics[width=0.7\linewidth]{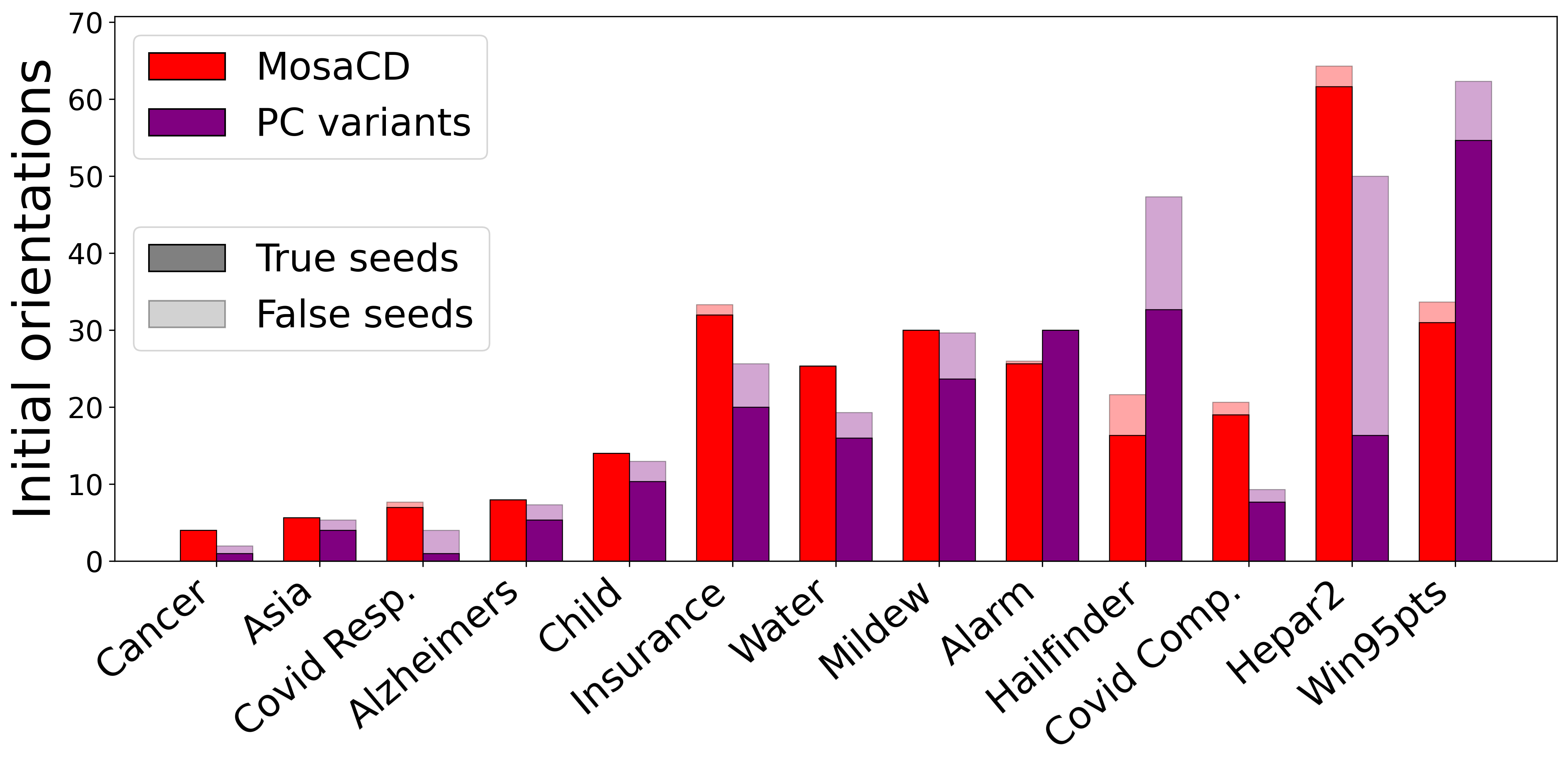}
    \caption{\textbf{Accuracy of orientation seeds.} Number of true and false directions found by \method{} LLM-based orientation seeding (Step 2) and standard PC v-structure orientation, averaged for each dataset across PC, PC-stable, and CPC skeletons.}
    \label{fig:orientation_seed}
\end{figure}

We further evaluated the effectiveness of \method{}'s LLM-based orientation seeding (Step 2) compared to the standard PC procedure (orienting v-structures). 
Results are reported in Figure \ref{fig:orientation_seed}, averaged across PC, PC-Stable, and CPC skeletons. 
First, \method{}'s seeding procedure identified substantially more true directions (avg \method{} vs. PC true seed ratio 2.15) and markedly fewer false directions (avg 5.0\% vs. 27.1\%) than PC across the 13 datasets. 
Second, in the 4 largest datasets where the memorization baseline underperformed, \method{} remained robust: it did not introduce excessive false seeds compared to PC (average 11.2\% vs. 32.1\% across such datasets), likely due to the hallucination filtering procedure. 
Full results, cost analysis, and additional ablations appear in Appendices \ref{app:orient}--\ref{app:sample}.

\subsection{Secondary analyses and ablation studies}

\begin{figure}[h]
    \centering
    \includegraphics[width=0.68\linewidth]{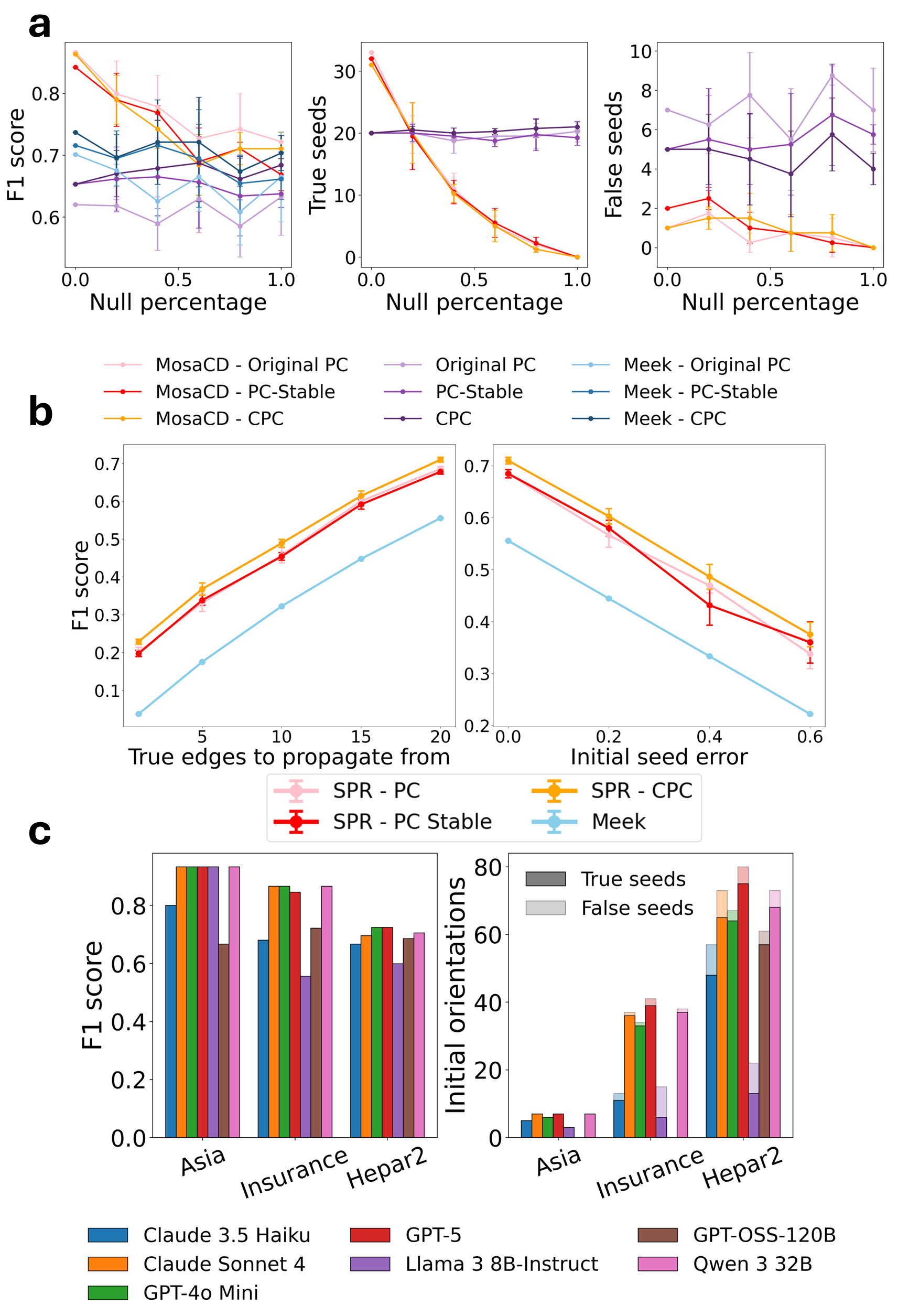}
    \caption{\textbf{Secondary analysis.}
    \textbf{(a)} F1 score, true seeds, and false seeds for \method{}, PC, and Meek as uninformative variable descriptions increase, using PC, PC-Stable, and CPC skeletons on Insurance.
    \textbf{(b)} F1 score for SPR and Meek on Insurance across PC, PC-Stable, and CPC skeletons: left, varying true seeds with no false seeds; right, varying false-seed proportion with 20 total seeds.
    \textbf{(c)} F1 scores and LLM correctness across seven LLMs on Asia, Insurance, and Hepar2 using PC skeletons.
    }
    \label{fig:secondary}
\end{figure}
First, we assessed the robustness of \method{}’s LLM inference by replacing a proportion of variable names and descriptions in the ``Insurance'' dataset with uninformative ones. 
Results are reported in Figure \ref{fig:secondary}a.
While \method{}’s performance and the number of true orientation seeds declined expectedly as the proportion of uninformative variables increased, it consistently generated only a small number of false seeds and remained competitive or superior overall to the baseline, demonstrating robustness and consistent with SPR's theory. Experiments in Asia and Hepar2 yield similar results (Appendix \ref{app:null_sims}). Graph-independent aliases (Figure~\ref{fig:rebuttal:aliases}) and post-cutoff comparisons appear in Appendix~\ref{app:rebuttal:aliases}.


Second, we evaluated the robustness of SPR through ablation studies on the ``Insurance'' dataset. We vary the amount and accuracy of external domain knowledge to propagate from, comparing against Meek's propagation rules starting from the same seeds.
Results are reported in Figure \ref{fig:secondary}b.
SPR consistently outperformed Meek both when varying the number of true seeds (with false seeds fixed at 0) and when varying the proportion of false seeds (with total seeds fixed at 20), demonstrating its effectiveness in utilizing such information. We repeat this analysis in the Asia and Hepar2 datasets and achieve similar results (Appendix \ref{app:propagation}).


Third, we varied the LLM backbone in \method{} across efficient models (Claude-3.5-Haiku, GPT-4o-mini), frontier reasoning models (GPT-5, Claude-Sonnet-4), and open-source models (GPT-oss-120b, Llama 3 8B-instruct, Qwen 3 32B) on 3 representative datasets: Asia, Insurance, and Hepar2.
As shown in Figure~\ref{fig:secondary}c, \method{} remains consistent across LLMs, except for slightly weaker performance with GPT-oss-120b and Llama 3 8B-instruct. A seven-sample-size study with learned-skeleton ceilings appears in Figure~\ref{fig:rebuttal:samples} (Appendix~\ref{app:rebuttal:samples}).


\bibliography{sample}


\newpage
\appendix
\onecolumn

\section{Limitations}

MosaCD inherits several limitations from constraint-based causal discovery. Its theoretical guarantees rely on standard assumptions such as causal sufficiency, Markovness, faithfulness, an accurate skeleton, correct separating-set records, and seed orientations that are compatible with the target equivalence class. In finite samples, violations of these assumptions can lead to incorrect skeleton edges, missing edges, or erroneous separating sets; while SPR is designed to reduce the impact of fragile collider-first propagation, it cannot fully correct errors introduced earlier in skeleton discovery.

The method also depends on the quality of external orientation information. Although our randomized prompting and filtering reduce positional bias and hallucinated LLM seeds, LLM-provided domain knowledge may still be wrong, sensitive to variable descriptions, affected by memorization, or inconsistent across model versions. In domains where variable names are uninformative, mechanisms are highly specialized, or causal relations are counterintuitive, MosaCD may provide fewer useful seeds and fall back closer to standard constraint-based behavior.

Our empirical evaluation covers 13 benchmark datasets with up to 76 nodes, additional 211-node MACH and 283-node BEELINE tasks, and a larger runtime study, but this does not exhaust the range of real-world causal discovery settings. In particular, we do not evaluate latent confounding, selection bias, feedback cycles, interventions, time-varying systems, or very high-dimensional graphs. Scaling also remains partly limited by the skeleton-search step and by the number of LLM queries, although the latter is easily parallelized. Finally, the optional vote-completion step produces a fully oriented DAG for evaluation, but it is not covered by the oracle SPR guarantee and should be interpreted as a pragmatic completion heuristic rather than a theoretically guaranteed recovery step.

\newpage
\section{Results with different PC skeletons}
\label{app:main}
\begin{table*}[htb!]
\footnotesize
\setlength{\tabcolsep}{2.5pt}
\centering

\begin{tabular}{lccccccc}
\hline
\multicolumn{1}{c}{} & \multicolumn{7}{c}{F1/SHD} \\
\hline
 & Cancer (5) & Asia (8) & Covid Resp. (11) & Alzheimers (11) & Child (20) & Insurance (27) & Water (32) \\
\hline
PC-Stable & 0.50/2 & 0.67/3 & 0.32/17 & 0.44/14 & 0.70/11 & 0.63/28 & 0.48/51 \\
CPC & 0.50/2 & 0.75/4 & 0.32/17 & 0.47/14 & 0.90/3 & 0.63/28 & 0.47/52 \\
PC-Stable + Meek & 0.50/2 & 0.67/3 & 0.40/14 & 0.48/12 & 0.78/6 & 0.69/20 & 0.58/39 \\
CPC + Meek & 0.50/2 & 0.93/1 & 0.40/14 & 0.55/11 & 0.90/3 & 0.72/19 & 0.57/40 \\
PC-Stable + MosaCD & 1.00/0 & 0.93/1 & 0.65/11 & 0.73/8 & 0.90/3 & 0.84/13 & 0.65/11 \\
CPC + MosaCD & 1.00/0 & 0.93/1 & 0.69/10 & 0.73/8 & 0.94/2 & 0.86/12 & 0.61/36 \\
\hline
\end{tabular}

\vspace{0.5em}

\begin{tabular}{lcccccc}
\hline
 & Mildew (35) & Alarm (37) & Hailfinder (56) & Covid. Comp. (63) & Hepar2 (70) & Win95pts (76) \\
\hline
PC-Stable & 0.69/21 & 0.85/12 & 0.38/101 & 0.18/146 & 0.40/92 & 0.64/64 \\
CPC & 0.64/23 & 0.85/12 & 0.38/108 & 0.17/150 & 0.42/89 & 0.64/64 \\
PC-Stable + Meek & 0.73/16 & 0.90/8 & 0.40/98 & 0.27/123 & 0.43/82 & 0.69/53 \\
CPC + Meek & 0.71/17 & 0.90/8 & 0.39/99 & 0.27/124 & 0.44/83 & 0.69/53 \\
PC-Stable + MosaCD & 0.87/10 & 0.93/5 & 0.56/81 & 0.39/118 & 0.72/53 & 0.74/42 \\
CPC + MosaCD & 0.87/10 & 0.91/6 & 0.47/91 & 0.40/117 & 0.73/52 & 0.79/37 \\
\hline
\end{tabular}

\caption{Benchmark evaluation. Each entry reports F1/SHD for each dataset and method (nodes in brackets).}
\label{tab:comparison_appendix}
\end{table*}

\newpage

\section{Orientation correctness by skeleton}
\label{app:orient}

\begin{figure}[H]
    \centering
    \includegraphics[width=0.75\linewidth]{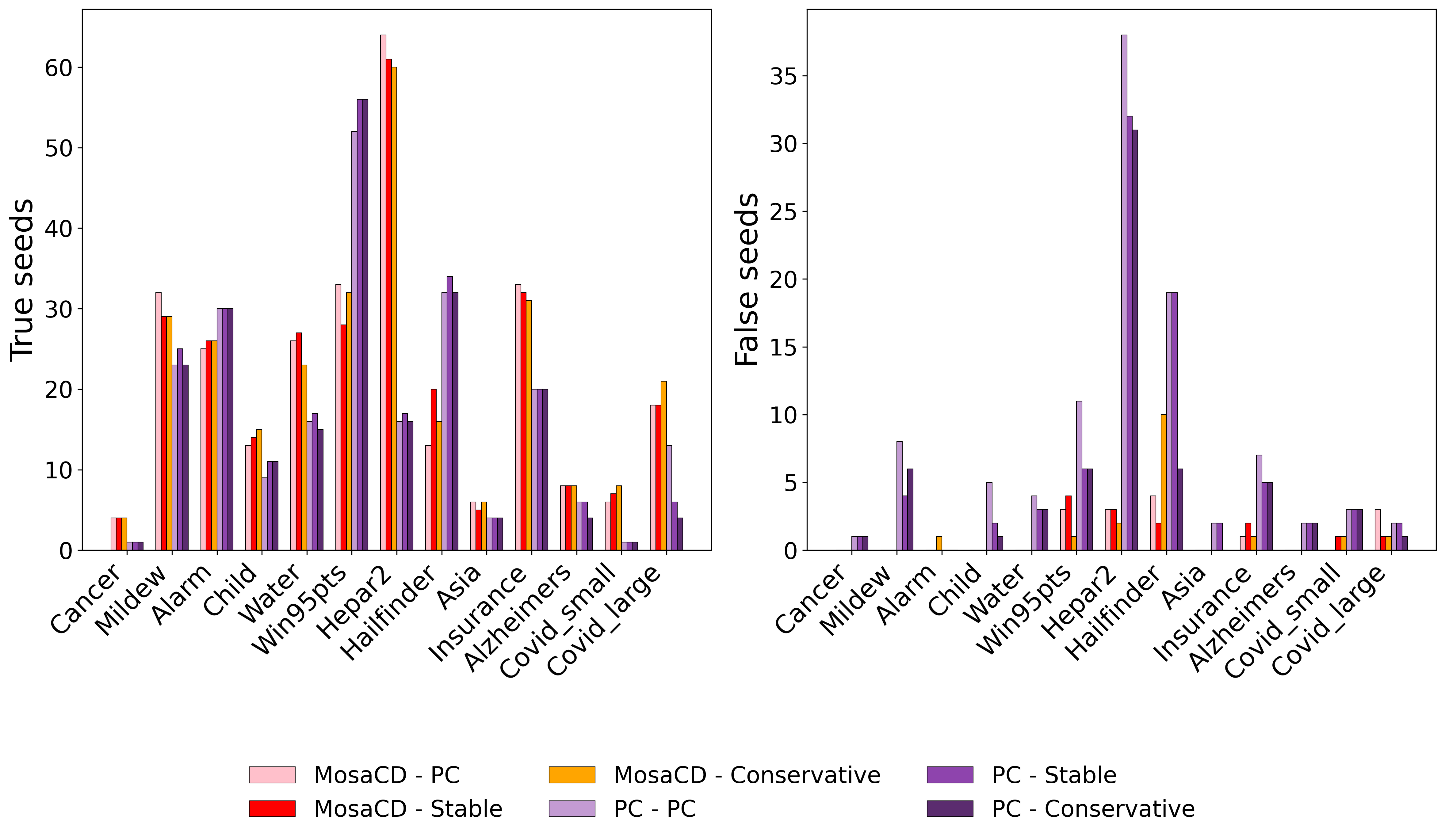}
    \caption{\textbf{Initial orientation correctness.} The average number of true and false seeds discovered by \method{} and PC variants in each dataset, aggregated across PC, PC-stable, and CPC.}
    \label{fig:orientation}
\end{figure}

\newpage
\section{{Positional bias}}
\label{app:bias}

\begin{figure}[H]
    \centering
    \includegraphics[width=1.0\linewidth]{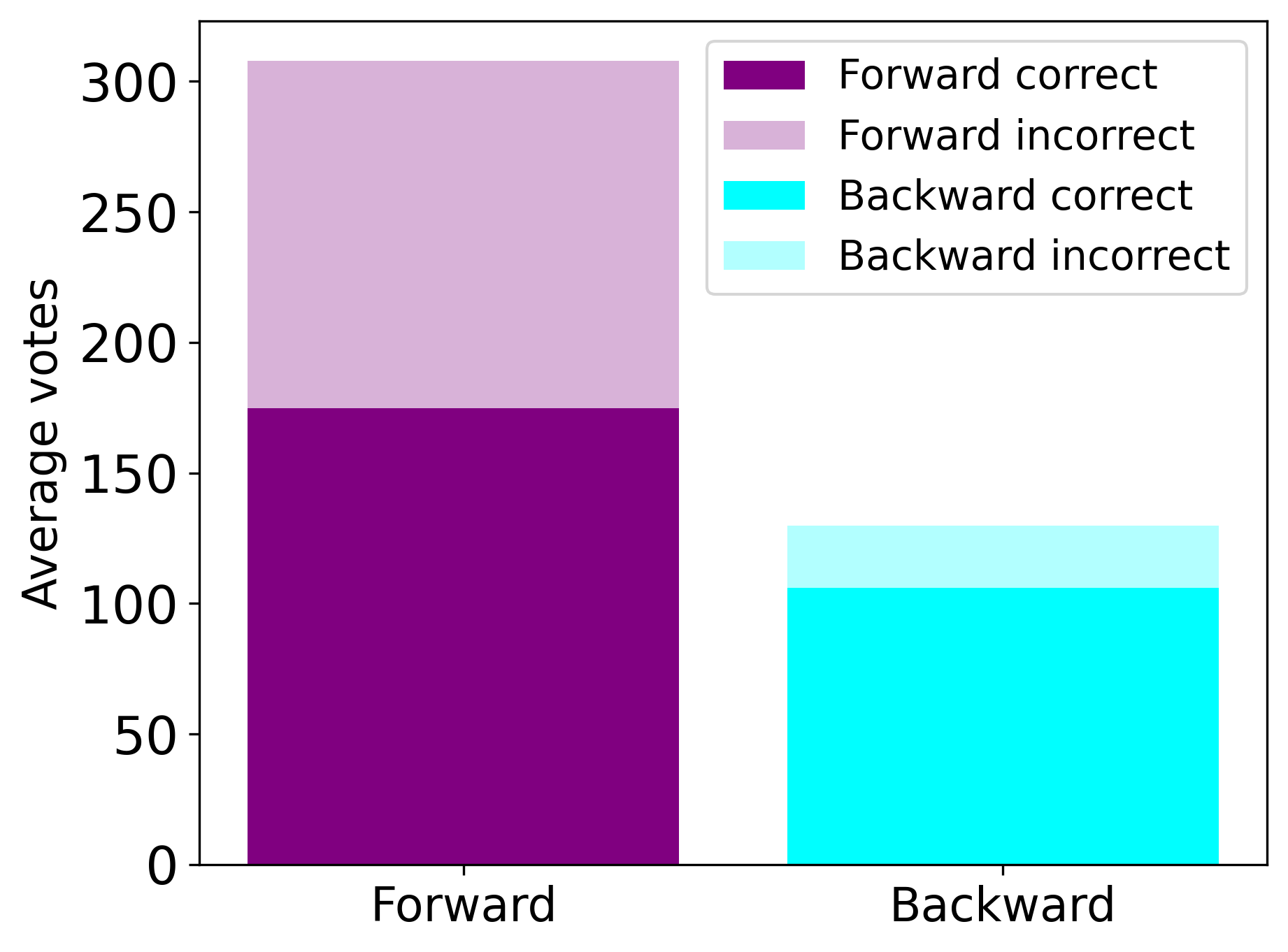}
    \caption{\textbf{Average count of total votes in each direction for a given prompt.} Forward appears before backward, and thus is chosen more frequently by the LLM. Shaded bars indicate votes in the wrong direction. Of note, I don't know was the first option, which was never selected.}
    \label{fig:positional_bias}
\end{figure}

\newpage

\section{{Runtime}}
\label{app:runtime}

Here we retain the original stage-timing measurements in seconds for skeleton search, LLM voting, and propagation. These measurements are separate from the comparative resource analyses in Appendix~\ref{app:rebuttal:resources}; accumulated query time and parallel wall-clock time should not be interchanged. The original larger runtime case, Andes (223 nodes), required 15 minutes of skeleton searching and 120 minutes of LLM voting. Reported end-to-end timings and API costs for the additional MACH and BEELINE tasks appear in Appendix~\ref{app:rebuttal:scale}.

\begin{table}[h!]
\centering
\small
\setlength{\tabcolsep}{3pt}
\begin{tabular}{l|ccc|ccc|ccc}
\hline
 & \multicolumn{3}{c|}{\textbf{Asia} (8 nodes)} 
 & \multicolumn{3}{c|}{\textbf{Insurance} (27 nodes)} 
 & \multicolumn{3}{c}{\textbf{Hepar2} (70 nodes)} \\
\cline{2-10}
 & PC & PC-Stable & CPC 
 & PC & PC-Stable & CPC
 & PC & PC-Stable & CPC \\
\hline
Skeleton     & 0.23 & 0.50 & 0.55 & 50.32 & 76.51 & 75.89 & 187.05 & 313.21 & 264.50 \\
LLM          & 256.82 & 232.13 & 264.14 & 1501.94 & 1397.03 & 1434.65 & 2492.51 & 2453.28 & 2260.18 \\
Propagation  & 0.00 & 0.00 & 0.00 & 0.01 & 0.01 & 0.01 & 0.03 & 0.03 & 0.03 \\
\hline
\end{tabular}
\caption{\textbf{Runtime analysis.} Runtime (seconds) of each stage of MosaCD for each dataset and PC variant.}
\end{table}

\newpage
\section{Ablating least-conflict orientation}
\label{app:conflict}

\begin{table}[H]
\centering
\begin{tabular}{lcc}
\toprule
Dataset & MosaCD & Randomized conflict \\
\midrule
Cancer          & 1.00 & 1.00 \\
Asia            & 0.93 & 0.93 \\
Covid Resp.     & 0.65 & 0.52 \\
Alzheimer's     & 0.73 & 0.73 \\
Child           & 0.90 & 0.86 \\
Insurance       & 0.87 & 0.87 \\
Water           & 0.59 & 0.59 \\
Mildew          & 0.90 & 0.90 \\
Alarm           & 0.93 & 0.93 \\
Hailfinder      & 0.49 & 0.47 \\
Covid. Comp.    & 0.39 & 0.39 \\
Hepar2          & 0.72 & 0.72 \\
Win95pts        & 0.81 & 0.77 \\
\bottomrule
\end{tabular}
\caption{F1 scores for the conflict-resolution ablation comparing MosaCD with a randomized conflict tie-breaker.}
\label{tab:randomized-conflict-ablation}
\end{table}

\newpage
\section{Ablating variable information}
\label{app:description}

\begin{table}[H]
\centering
\begin{tabular}{lcc}
\toprule
Dataset & Names + descriptions & Names only \\
\midrule
Cancer          & 1.00 & 1.00 \\
Asia            & 0.93 & 0.93 \\
Covid Resp.     & 0.65 & 0.65 \\
Alzheimer's     & 0.73 & 0.73 \\
Child           & 0.90 & 0.83 \\
Insurance       & 0.87 & 0.87 \\
Water           & 0.59 & 0.57 \\
Mildew          & 0.90 & 0.89 \\
Alarm           & 0.93 & 0.93 \\
Hailfinder      & 0.49 & 0.44 \\
Covid. Comp.    & 0.39 & 0.39 \\
Hepar2          & 0.72 & 0.72 \\
Win95pts        & 0.81 & 0.74 \\
\bottomrule
\end{tabular}
\caption{F1 scores for the variable-description ablation, when \method{} sees variable names and descriptions, or variable names only.}
\label{tab:names-only-ablation}
\end{table}

\newpage
\section{Ablating logprob-based confidence}
\label{app:logprobs}

\begin{table}[H]
\centering
\begin{tabular}{lcc}
\toprule
Dataset & MosaCD & Logprobs \\
\midrule
Cancer          & 1.00 & 1.00 \\
Asia            & 0.93 & 0.93 \\
Covid Resp.     & 0.65 & 0.32 \\
Alzheimer's     & 0.73 & 0.53 \\
Child           & 0.90 & 0.56 \\
Insurance       & 0.87 & 0.61 \\
Water           & 0.59 & 0.53 \\
Mildew          & 0.90 & 0.30 \\
Alarm           & 0.93 & 0.33 \\
Hailfinder      & 0.49 & 0.33 \\
Covid. Comp.    & 0.39 & 0.27 \\
Hepar2          & 0.72 & 0.54 \\
Win95pts        & 0.81 & 0.63 \\
\bottomrule
\end{tabular}
\caption{F1 scores for the logprob-based confidence ablation, comparing \method{} with a variant using logprobs instead of randomized shuffling and voting.}
\label{tab:logprobs-ablation}
\end{table}

\newpage
\section{Varying sample size}
\label{app:sample}

\begin{figure}[H]
    \centering
    \includegraphics[width=1\linewidth]{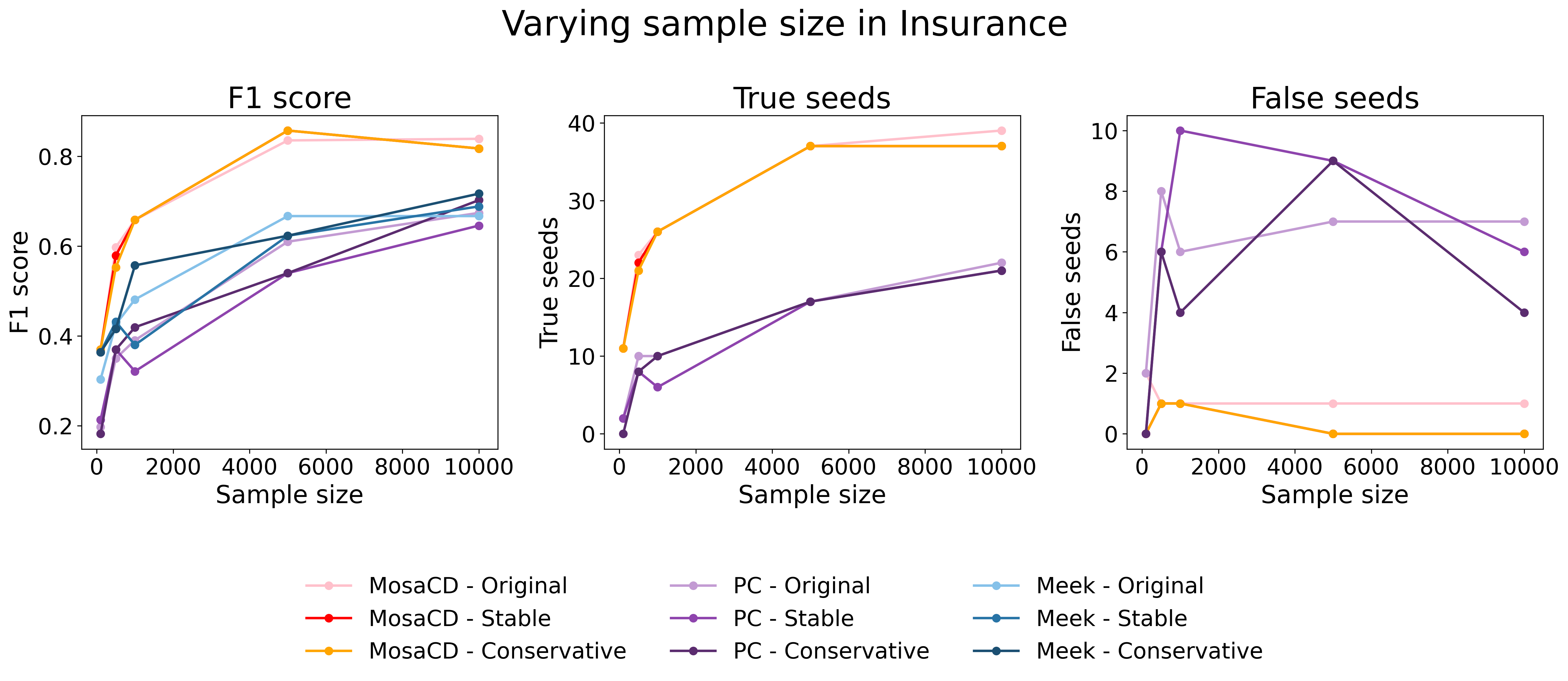}
    \caption{\textbf{Varying sample size.} Results showing F1 score, true seed counts, and false seed counts for MosaCD, PC, and Meek's rules for sample size from 100 to 10000.}
    \label{fig:sample_size}
\end{figure}

\newpage

\section{Additional null simulations}
\label{app:null_sims}

\begin{table}[H]
\centering
\begin{tabular}{lccccc}
\toprule
Dataset & PC & Meek & MosaCD (0\% null) & MosaCD (40\% null) & MosaCD (100\% null) \\
\midrule
Asia   & 0.67 & 0.67 & 0.93 & 0.81 & 0.67 \\
Hepar2 & 0.36 & 0.39 & 0.72 & 0.60 & 0.42 \\
\bottomrule
\end{tabular}
\caption{F1 scores for the null information experiment on Asia and Hepar2.}
\label{tab:null-info-asia-hepar2}
\end{table}

\newpage

\section{Propagation ablations}
\label{app:propagation}

\begin{figure}[H]
    \centering
    \includegraphics[width=1.0\linewidth]{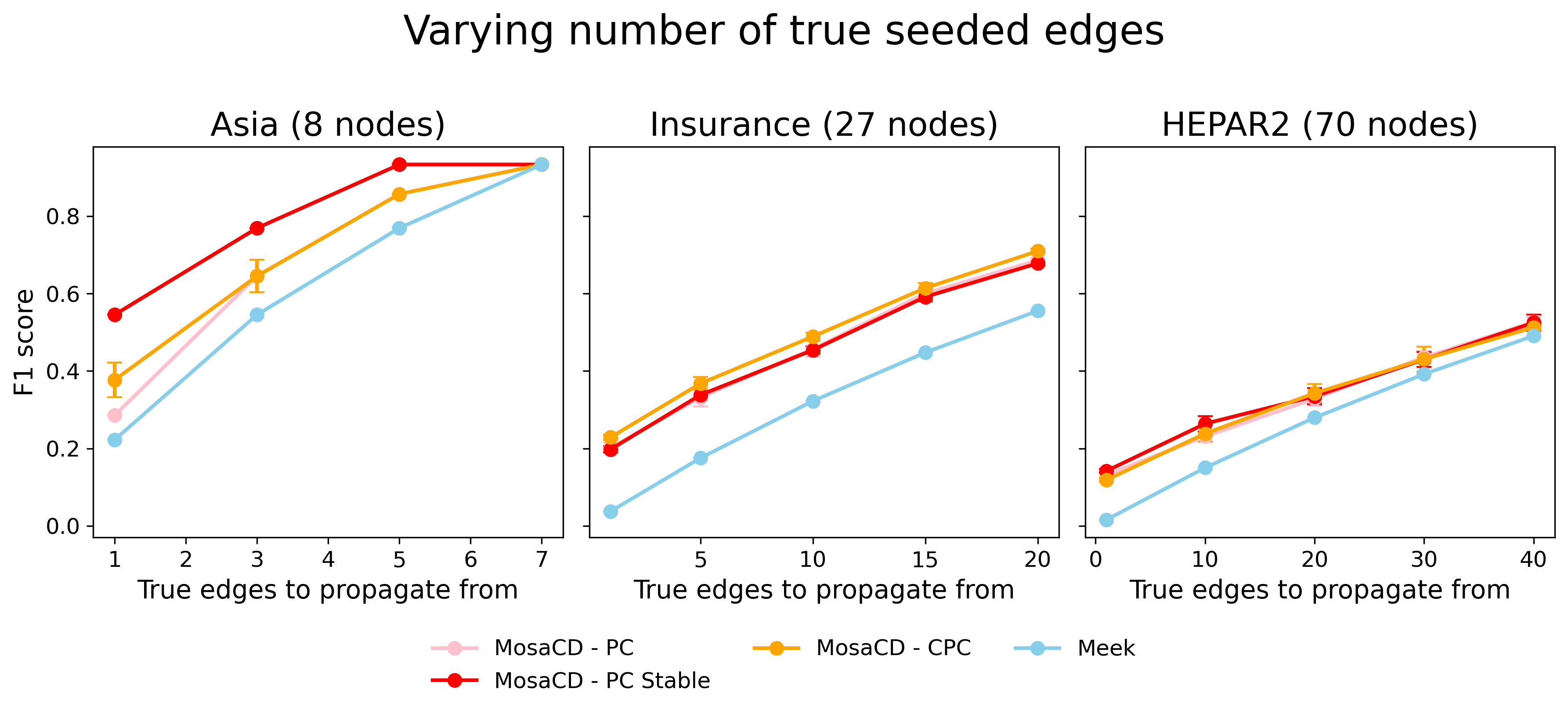}
    \caption{\textbf{Experiments varying number of true seeds.} F1 score for SPR (using PC,
PC-Stable, and CPC skeletons) and Meek.}
    \label{fig:propagate_true}
\end{figure}
\begin{figure}[H]
    \centering
    \includegraphics[width=1.0\linewidth]{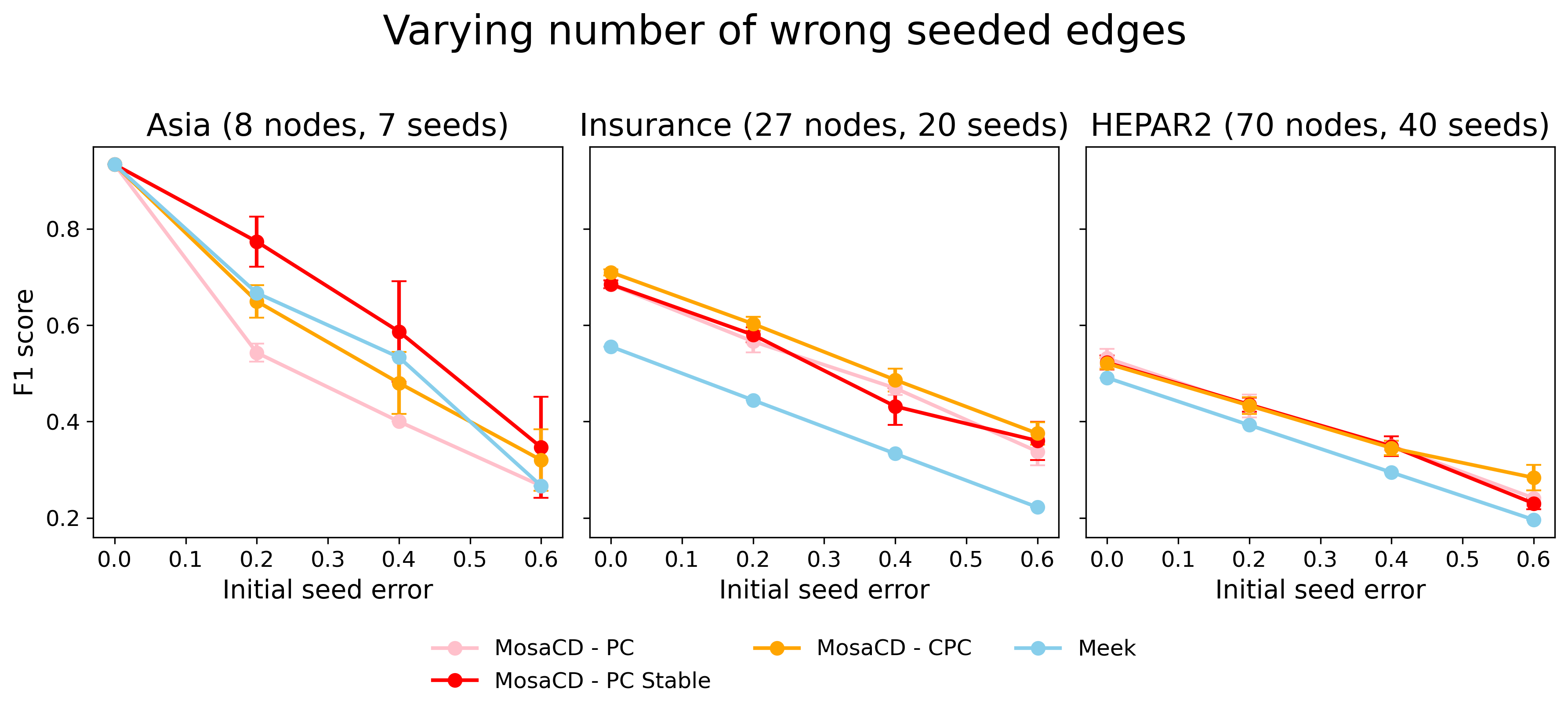}
    \caption{\textbf{Experiments varying number of false seeds.} F1 score for SPR (using PC,
PC-Stable, and CPC skeletons) and Meek.}
    \label{fig:propagate_false}
\end{figure}

\newpage

\section{Additional experiments and reproducibility details}
\label{app:rebuttal}

\subsection{Separating propagation from optional completion}
\label{app:rebuttal:stages}

Table~\ref{tab:rebuttal:stages} and Figure~\ref{fig:rebuttal:stages} distinguish accepted LLM seeds, SPR closure, least-conflict orientation, and optional vote completion. Holding the accepted seeds fixed, SPR increases reported F1 from .75 to .93 on Asia, .78 to .85 on Insurance, and .66 to .69 on Hepar2. Optional completion then adds .00, .02, and .03, respectively. These empirical gains do not extend the oracle guarantee of Theorem~\ref{thrm:correctness} to incorrect seeds or to completed DAGs. Observational CI information identifies an equivalence class under the stated assumptions; accepted external seeds restrict that class, and completion chooses residual directions for full-DAG evaluation.

\begin{table}[htbp]
\centering
\small
\begin{tabular}{lccc}
\toprule
Stage & Asia & Insurance & Hepar2 \\
\midrule
Skeleton & .38 & .36 & .32 \\
+ cross-order consistency & .75 & .78 & .67 \\
+ CI/cycle filtering & .75 & .78 & .66 \\
+ SPR closure & .93 & .85 & .69 \\
+ least-conflict & .93 & .85 & .69 \\
+ optional vote completion & .93 & .87 & .72 \\
\bottomrule
\end{tabular}
\caption{\textbf{Stage-wise reported F1.} Each row is the output after the named stage on Asia, Insurance, or Hepar2. The SPR and least-conflict rows precede optional full-DAG completion.}
\label{tab:rebuttal:stages}
\end{table}

\begin{figure}[htbp]
\centering
\includegraphics[width=\linewidth]{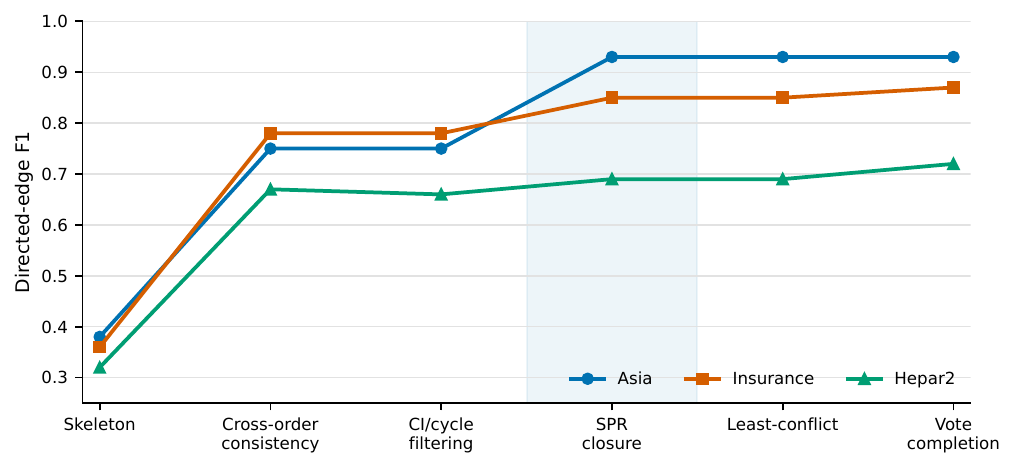}
\caption{\textbf{Contribution of propagation and completion.} Reported F1 at successive pipeline stages for Asia (circles), Insurance (squares), and Hepar2 (triangles), using the numerical values in Table~\ref{tab:rebuttal:stages}. The shaded band marks SPR closure. The comparison after CI/cycle filtering holds the accepted seed set fixed; least-conflict orientation does not change F1 in these three cases. Lines connect stage measurements and do not imply a continuous or monotone process.}
\label{fig:rebuttal:stages}
\end{figure}

\subsection{Graph-independent aliases and benchmark familiarity}
\label{app:rebuttal:aliases}

We replace names with graph-independent animal, color, or atom aliases, using five replicates of each alias type on Asia, Insurance, and Hepar2 and applying the same intervention to every LLM baseline. Table~\ref{tab:rebuttal:aliases} and Figure~\ref{fig:rebuttal:aliases} report the original-label and alias results. MosaCD falls toward the PC/Meek/SPR range when useful semantics are removed. The null-information experiment in Appendix~\ref{app:null_sims} is a separate diagnostic and is not pooled with these alias measurements.

\begin{table}[htbp]
\centering
\small
\begin{tabular}{lccc}
\toprule
Method & Asia & Insurance & Hepar2 \\
\midrule
MosaCD & $.93\to.67$ & $.87\to.65$ & $.72\to.36$ \\
Memorization & $.71\to.22$ & $.27\to.05$ & $.11\to.05$ \\
SCP & $.67\to.67$ & $.68\to.60$ & $.36\to.25$ \\
Vashishtha & $.93\to.54$ & $.77\to.30$ & $.70\to.25$ \\
Jiralerspong & $.46\to.22$ & $.16\to.05$ & $.16\to.05$ \\
LLM-CD & $.94\to.62$ & $.22\to.22$ & $.18\to.16$ \\
Causal Disco & $.86\to.46$ & $.65\to.22$ & $.92\to.25$ \\
ILS-CSL & $.67\to.62$ & $.76\to.58$ & $.51\to.36$ \\
\midrule
PC / Meek / SPR & .58 / .67 / .67 & .53 / .60 / .65 & .26 / .36 / .37 \\
\bottomrule
\end{tabular}
\caption{\textbf{Reported F1 with graph-independent aliases.} Each arrow gives original-label F1 followed by random-alias F1. The final row gives the statistical controls reported for this experiment. The alias measurements are separate from the main benchmark and null-information runs; their baseline values should not be interchanged. Five replicates per alias type were performed.}
\label{tab:rebuttal:aliases}
\end{table}

\begin{figure}[htbp]
\centering
\includegraphics[width=\linewidth]{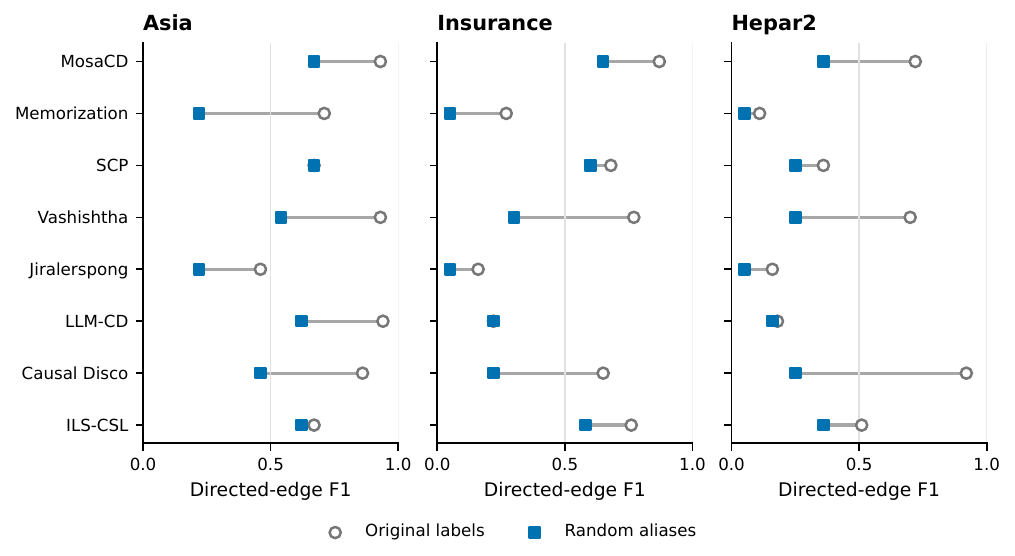}
\caption{\textbf{Effect of removing informative variable names.} Open circles denote original-label F1 and blue squares denote random-alias F1 for each method, using Table~\ref{tab:rebuttal:aliases}. Each line connects the two reported aggregate values for one method and dataset. Aliases are readable animal, color, or atom labels assigned independently of the graph. All methods receive the same naming intervention. This figure quantifies changes in F1, not false-positive counts or a causal attribution of the change to memorization.}
\label{fig:rebuttal:aliases}
\end{figure}

The three post-cutoff datasets in the original benchmark are Covid Resp., ALZ, and Covid Comp. MACH adds a fourth post-cutoff evaluation. Table~\ref{tab:rebuttal:postcutoff} compares these tasks with whole-graph LLM controls and PC-family methods. 

\begin{table}[htbp]
\centering
\small
\begin{tabular}{lcccc}
\toprule
Method & Covid Resp. & ALZ & Covid Comp. & MACH \\
\midrule
MosaCD & .65 & .73 & .39 & .40 \\
Vashishtha & .65 & .48 & .27 & .32 \\
Memorization & .55 & .55 & .17 & .22 \\
Best PC-family & .32 & .53 & .28 & .25 \\
\bottomrule
\end{tabular}
\caption{\textbf{Reported F1 on post-cutoff tasks.} The first three tasks belong to the original 13-dataset comparison; MACH is an additional evaluation. ``Best PC-family'' denotes the best reported PC, Meek, or SPR result for each task.}
\label{tab:rebuttal:postcutoff}
\end{table}

\subsection{Additional evaluations at 211 and 283 nodes}
\label{app:rebuttal:scale}

We evaluate a 211-node streamflow task derived from MACH \citep{sink2026mach} and a 283-node mHSC-GM gene-regulatory task derived from BEELINE \citep{pratapa2020beeline}.  Table~\ref{tab:rebuttal:scale} gives the reported F1, end-to-end runtime, and API cost for the methods evaluated. MosaCD improves over the best tested PC-family F1 from .25 to .40 on MACH and .06 to .18 on BEELINE.

\begin{table}[htbp]
\centering
\small
\setlength{\tabcolsep}{5pt}
\begin{tabular}{lrrr|rrr}
\toprule
& \multicolumn{3}{c|}{MACH (211 nodes)} & \multicolumn{3}{c}{BEELINE (283 nodes)} \\
Method & F1 & Time (s) & USD & F1 & Time (s) & USD \\
\midrule
MosaCD & .40 & 530 & .70 & .18 & 761 & 1.18 \\
Memorization & .22 & 25 & .01 & .08 & 33 & .01 \\
Vashishtha & .32 & 987 & .13 & .05 & 2074 & .25 \\
SCP & .29 & 85 & .01 & .06 & 28 & .01 \\
PC & .25 & 501 & --- & .06 & 690 & --- \\
Meek & .24 & 505 & --- & .06 & 705 & --- \\
SPR & .25 & 503 & --- & .06 & 711 & --- \\
\bottomrule
\end{tabular}
\caption{\textbf{Additional larger-graph evaluations.} Time is reported end-to-end runtime in seconds; USD is the reported LLM API cost, with a dash for methods without an LLM API cost. Only the methods shown were reported in these evaluations. These measurements are separate from the original stage-timing study in Appendix~\ref{app:runtime} and the small-graph resource table below.}
\label{tab:rebuttal:scale}
\end{table}

\subsection{Finite-sample degradation and the skeleton ceiling}
\label{app:rebuttal:samples}

We vary sample size from 100 to 10,000 in Insurance and Hepar2, thereby jointly changing the quality of the learned skeleton and separating-set record. The skeleton ceiling is the maximum directed-edge F1 attainable if every retained true skeleton edge is oriented correctly. MosaCD exceeds Meek at every reported sample size in both datasets, while remaining bounded by the learned skeleton.

\begin{figure}[htbp]
\centering
\includegraphics[width=\linewidth]{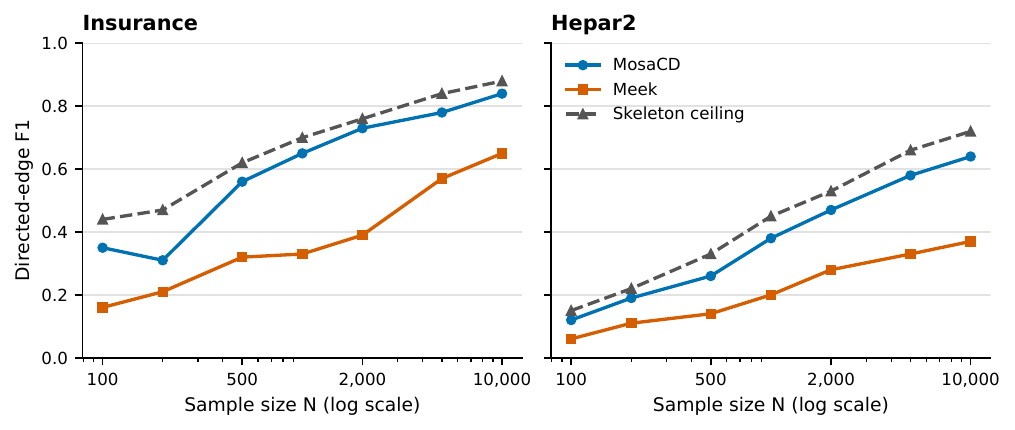}
\caption{\textbf{Finite-sample performance relative to the learned skeleton.} Reported F1 for MosaCD (blue circles), Meek (orange squares), and the skeleton ceiling (dashed gray triangles) on Insurance and Hepar2 at $N\in\{100,200,500,1000,2000,5000,10000\}$. The horizontal axis is logarithmic. The ceiling conditions on the learned skeleton and assumes correct directions for its retained true edges. Lines connect the reported sample sizes. Exact values appear in Table~\ref{tab:rebuttal:samples}.}
\label{fig:rebuttal:samples}
\end{figure}

\begin{table}[htbp]
\centering
\small
\begin{tabular}{rrrrrrr}
\toprule
& \multicolumn{3}{c}{Insurance} & \multicolumn{3}{c}{Hepar2} \\
$N$ & MosaCD & Meek & Ceiling & MosaCD & Meek & Ceiling \\
\midrule
100 & .35 & .16 & .44 & .12 & .06 & .15 \\
200 & .31 & .21 & .47 & .19 & .11 & .22 \\
500 & .56 & .32 & .62 & .26 & .14 & .33 \\
1,000 & .65 & .33 & .70 & .38 & .20 & .45 \\
2,000 & .73 & .39 & .76 & .47 & .28 & .53 \\
5,000 & .78 & .57 & .84 & .58 & .33 & .66 \\
10,000 & .84 & .65 & .88 & .64 & .37 & .72 \\
\bottomrule
\end{tabular}
\caption{\textbf{Reported sample-size sensitivity.} Entries are directed-edge F1. ``Ceiling'' is the maximum attainable F1 conditional on the learned skeleton, as defined in the text. }
\label{tab:rebuttal:samples}
\end{table}

\subsection{A matched one-oracle-edge diagnostic}
\label{app:rebuttal:oracle}

To separate propagation from access to additional ground-truth information, we supply the identical sampled correct edge to each compatible non-LLM method. Table~\ref{tab:rebuttal:oracle} reports aggregate F1 across the original 13 datasets. SPR with one oracle edge reaches .57, compared with .54 for PC and Meek with the same supplied edge.

\begin{table}[htbp]
\centering
\small
\begin{tabular}{lcc}
\toprule
Method & Baseline F1 & With one oracle edge \\
\midrule
PC & .48 & .54 \\
Meek & .52 & .54 \\
NOTEARS & .23 & .28 \\
GES & .38 & .40 \\
SPR & .53 & .57 \\
\bottomrule
\end{tabular}
\caption{\textbf{Matched access to one correct orientation.} The same sampled ground-truth edge is supplied to every compatible method. Values are the reported aggregate F1 across 13 datasets.}
\label{tab:rebuttal:oracle}
\end{table}

\subsection{Comparative computational resources and baseline interfaces}
\label{app:rebuttal:resources}

All LLM methods use the same backbone and variable metadata, but otherwise follow their published interfaces. MosaCD's CI context, two-order voting, filtering, and SPR are components of its algorithm, rather than additions made to every baseline. The comparison therefore does not match prompt budgets or CI access across methods. Tables~\ref{tab:rebuttal:resources} and \ref{tab:rebuttal:nonllm} give the reported small-graph measurements. 

\begin{table}[htbp]
\centering
\small
\setlength{\tabcolsep}{5pt}
\begin{tabular}{lrr|rr|rr}
\toprule
& \multicolumn{2}{c|}{Asia} & \multicolumn{2}{c|}{Insurance} & \multicolumn{2}{c}{Hepar2} \\
Method & Time (s) & USD & Time (s) & USD & Time (s) & USD \\
\midrule
Memorization & 1.80 & .01 & 4.00 & .01 & 12.60 & .01 \\
SCP & .20 & .00 & 52.90 & .01 & 247.50 & .01 \\
Vashishtha & 18.60 & .01 & 186.70 & .01 & 507.00 & .02 \\
Jiralerspong & 16.60 & .01 & 72.50 & .01 & 82.10 & .01 \\
LLM-CD & 55.10 & .01 & 812.40 & .01 & 1291.90 & .01 \\
Causal Disco & 15.30 & .01 & 106.80 & .01 & 270.70 & .03 \\
ILS-CSL & 95.60 & .01 & 620.60 & .05 & 1346.30 & .11 \\
MosaCD & 10.60 & .01 & 56.20 & .08 & 250.00 & .15 \\
\bottomrule
\end{tabular}
\caption{\textbf{Reported runtime and API cost for LLM methods.} Times and costs are transcribed from the comparative small-graph analysis. API costs are reported to two decimal places. These are parallelized measurements across all queries and thus differ from the sequential stage timings in Appendix~\ref{app:runtime}.}
\label{tab:rebuttal:resources}
\end{table}

\begin{table}[htbp]
\centering
\small
\begin{tabular}{lrrr}
\toprule
Method & Asia (s) & Insurance (s) & Hepar2 (s) \\
\midrule
PC & .18 & 46.08 & 237.52 \\
Meek & .18 & 46.08 & 237.53 \\
SPR & .18 & 46.09 & 237.54 \\
Shapley-PC & .43 & 347.84 & $>7200$ \\
NOTEARS & 1.83 & 72.03 & 25.07 \\
DAG-GNN & 26.30 & $>7200$ & $>7200$ \\
GES & .42 & 25.70 & 87.71 \\
Hill climbing & .41 & 26.86 & $>7200$ \\
MMHC & 101.30 & $>7200$ & $>7200$ \\
\bottomrule
\end{tabular}
\caption{\textbf{Reported non-LLM wall-clock runtimes.} Entries are seconds for the three small-graph tasks. Values above 7,200 indicate runs exceeding two hours.}
\label{tab:rebuttal:nonllm}
\end{table}

\clearpage

\section{Proofs for \ref{subsec_main:correctness}\label{appendix:correctness}}
\begin{lemma}[Collider/non-collider soundness w.r.t. recorded minimal separators]\label{correctness:lemma1}
Consider an unshielded triple $X-Z-Y$ with $X$ nonadjacent to $Y$. Under Markovness, faithfulness, and a perfect CI oracle, the recorded minimal separators satisfy:
\begin{enumerate}
\item if $X\to Z\leftarrow Y$ in the CPDAG, then $Z\notin S$ for every $S\in\sepsets(X,Y)$;
\item if the unshielded triple is not a collider in $G$ equivalently, in every DAG in $[G]$, then $Z\in S$ for every $S\in\sepsets(X,Y)$.
\end{enumerate}
Thus the all-absent and all-present tests are sound when the minimal-separator record is correct. If the recorded minimal separators contain both sets that include $Z$ and sets that exclude $Z$, SPR treats the triple as ambiguous and does not orient it using this rule.
\end{lemma}

\begin{proof}
    This is the standard separator characterization of unshielded triples under faithfulness: for nonadjacent $X$ and $Y$, the membership of the middle node $Z$ in minimal separating sets determines whether the unshielded triple is a collider or a non-collider. The final statement follows from the definition of SPR's CI-guarded rules, which require unanimity over $\sepsets(X,Y)$ and abstain under mixed evidence.
\end{proof}

\begin{lemma}[Soundness of CI-guarded rule closure]\label{correctness:lemma2}
    Start from the true skeleton after inserting an acyclic seed set $E_{\mathrm{seed}}=\bk$ compatible with $[G]_{\bk}$. Under the assumptions of Theorem~\ref{thrm:correctness}, every orientation added by SPR Step~1 preserves compatibility with $[G]_{\bk}$ and preserves acyclicity.
\end{lemma}

\begin{proof}
Step~3.1 first applies only the directed-path acyclicity rule: if $X-Y$ and $X\to\cdots\to Y$, then orient $X\to Y$, since the opposite orientation would create a directed cycle. Steps~3.2--3.3 are guarded by $\sepsets(X,Y)$ on unshielded triples. By Lemma~\ref{correctness:lemma1}, the all-present and all-absent separator tests are sound under the oracle assumptions and a correct recorded minimal-separator collection. Because candidate orientations that create directed cycles are rejected, each accepted orientation preserves acyclicity. By induction over the finite sequence of rule applications, all orientations added by SPR Step~1 remain compatible with the skeleton, $\sepsets$, and $\bk$.
\end{proof}

\paragraph{Proof of Theorem \ref{thrm:correctness}}
By the perfect CI oracle, faithfulness, and skeleton consistency assumptions, the skeleton search returns $\mathrm{Skel}(G)$ and the recorded minimal separator collection $\sepsets$ is correct. By seed compatibility, the initial PDAG after adding $E_{\mathrm{seed}}=\bk$ is acyclic and has at least one DAG extension in $[G]_{\bk}$. Lemma~\ref{correctness:lemma2} shows that each orientation added by SPR Step~1 preserves compatibility with the restricted equivalence class and with the recorded separator constraints. Since the graph is finite and orientations are only added, Step~1 terminates. The returned graph is therefore an acyclic PDAG with the true skeleton satisfying $\sepsets$ and $\bk$.

When $\bk=\emptyset$, the restricted class $[G]_{\bk}$ equals the ordinary Markov equivalence class $[G]$. Under full faithfulness and a perfect CI oracle, the recorded minimal separators correctly identify the unshielded colliders, and exhausting the standard acyclicity/orientation closure yields the ordinary CPDAG. In this oracle-faithful setting, PC, PC-stable, and CPC have the same skeleton and collider information and hence the same closure. This final equivalence is not asserted under adjacency-faithfulness alone.

\section{Proofs for \ref{subsec_main:accuracy}}
\label{appendix:accuracy}
Let $S_{Z,\ell}$ and $U_{Z,\ell}$ denote the numbers of true sepsets and true non-sepsets in $\text{incZ}_\ell$, and let $S_{\neg Z,\ell}$ and $U_{\neg Z,\ell}$ denote the same counts for $\text{notZ}_\ell$, and $S_\ell=S_{Z,\ell}+S_{\neg Z,\ell}$, $U_\ell=U_{Z,\ell}+U_{\neg Z,\ell}$. 

\begin{lemma}[First-hit filtration across levels]\label{accuracy:lemma1}
    Define event $D$ ``there is at least one independence hit at some level''. Define $F_\ell$ as the event ``no hit on levels $k<\ell$, and at least one hit on level $\ell$''. Under \ref{assumption:SimpleCI} (independent CI tests, FPR $=\alpha$, FNR $=\beta$), we have 
\begin{align*}
\Pr(D)&=1-\prod_{\ell}\beta^{S_\ell}(1-\alpha)^{U_\ell}\\
\mathrm{PrevNoHit}_\ell&:=\Pr(\text{no hit on }k<\ell)=\prod_{k<\ell}\beta^{S_k}(1-\alpha)^{U_k}\\
\Pr(F_\ell)&=\mathrm{PrevNoHit}_\ell\cdot\bigl(1-\beta^{S_\ell}(1-\alpha)^{U_\ell}\bigr)
\end{align*}
where $\{F_\ell\}$ is a disjoint partition of $D$.
\end{lemma}
\begin{proof}
    Under Assumption \ref{assumption:SimpleCI} each candidate on level $k$ hits (declares independence) with probability $1-\beta$ if it is a true sepset and with probability $\alpha$ if not; candidates and levels are independent. Thus ``no hit anywhere'' has probability $\prod_\ell \beta^{S_\ell}(1-\alpha)^{U_\ell}$, giving $\Pr(D)$. The rest is by independence across levels and the definition of $F_\ell$.
\end{proof}

\begin{assumption}[No short detours; $Z$-only control up to order $\ell$]\label{assumption:detours}
    Fix integers $\ell\ge0$ and $L\ge 2\ell+1$.
(a) (No short detours) Every $X\leadsto Y$ path of length $\le L$ contains $Z$.
(b) ($Z$-only control) For any conditioning set $C$ with $|C|\le \ell$, the segment $X-Z-Y$ is open/blocked iff $Z\notin C$/$Z\in C$ when $Z$ is a non-collider, and blocked/open iff $Z\notin C$/$Z\in C$ when $Z$ is a collider; in particular, conditioning on any other node (including descendants on the segment) cannot change the segment’s status for $|C|\le\ell$.
\end{assumption}

\begin{lemma}[Bucket counts and collider/non-collider separation]\label{accuracy:lemma2}
Assume Markov and Faithfulness and Assumption \ref{assumption:detours} (with $2\ell+1\le L$), at level $\ell=0$, if $Z$ is a non-collider, $S_{\neg Z,0}=0$, $U_{\neg Z,0}=1$; if $Z$ is a collider, $S_{\neg Z,0}=1$, $U_{\neg Z,0}=0$. In both cases, $\text{incZ}_0=\varnothing$, hence $S_{Z,0}=U_{Z,0}=0$. 

For all $\ell\ge1$, if non-collider truth, $S_{Z,\ell}=\text{incZ}_\ell$, $U_{Z,\ell}=0$, $S_{\neg Z,\ell}=0$, $U_{\neg Z,\ell}=\text{notZ}_\ell$; if collider truth, $S_{Z,\ell}=0$, $U_{Z,\ell}=\text{incZ}_\ell$, $S_{\neg Z,\ell}=\text{notZ}_\ell$, $U_{\neg Z,\ell}=0$.
\end{lemma}

\begin{proof}
D-separation rules are that, a non-collider blocks a path iff it is in $C$; a collider blocks unless it or a descendant is in $C$. 

When $\ell=0$, for a non-collider chain $X-Z-Y$, the path is open unconditionally, so $\emptyset$ is a non-sepset; for a collider, it is blocked unconditionally, so $\emptyset$ is a sepset.

When $\ell\ge1$: by Assumption \ref{assumption:detours}(a), all short $X-Y$ paths pass through $Z$; by Assumption \ref{assumption:detours}(b), with $|C|\le\ell$ the local segment’s status is controlled only by the inclusion of $Z$. So for a non-collider, if $Z\in C$ (bucket $\text{incZ}_\ell$), the local segment is blocked and, since every short path uses $Z$, all paths are blocked, thus sepset; if $Z\notin C$, the local segment is active and cannot be blocked by other vertices with $|C|\le\ell$, thus non-sepset. For a collider, if $Z\in C$, the local segment is opened and cannot be re-blocked by $|C|\le\ell$, thus non-sepset; if $Z\notin C$, the local segment remains blocked and no short detour exists, thus sepset.
\end{proof}

\begin{lemma}[Order-averaged first-hit factor within a level]\label{accuracy:lemma3}Fix a level containing $m$ true sepsets (each hits with prob. $a = 1-\beta$) and $n$ non-sepsets (each hits with prob. $b = \alpha$). Under a uniformly random within-level permutation (independent of outcomes by Assumption \ref{assumption:SimpleCI}), the average no-hit-from-predecessors factor for a fixed candidate equals
\begin{equation*}
I_{m,n}(a,b)=\int_0^1(1-au)^m(1-bu)^n\,du =\sum_{i=0}^m\sum_{j=0}^n\binom{m}{i}\binom{n}{j}\frac{(-a)^i(-b)^j}{i+j+1}
\end{equation*}
with $I_{m,0}=\frac{1-(1-a)^{m+1}}{a(m+1)}$ and $I_{0,n}=\frac{1-(1-b)^{n+1}}{b(n+1)}$.
\end{lemma}

\begin{proof}
Let $N:=m+n+1$ be the number of candidates on the level including a fixed target $k$. Write $p_j\in\{a,b\}$ for the hit probability of candidate $j\neq k$ and, for a random permutation $\pi$, define $X_\pi \;=\; \prod_{j\prec_\pi k}(1-p_j)$, the product of ``no-hit'' factors over predecessors of $k$. We average $X_\pi$ over all permutations.

If $r$ candidates precede $k$, then (i) $r$ is uniform on $\{0,\dots,N-1\}$ with probability $1/N$; (ii) conditional on $r$, the predecessor set $S$ is uniform over the $\binom{N-1}{r}$ subsets of $\{j\neq k\}$ of size $r$. Hence
$$
\mathbb{E}_\pi[X_\pi]=\frac{1}{N}\sum_{r=0}^{N-1}\frac{1}{\binom{N-1}{r}}
\sum_{\substack{S\subseteq\{j\neq k\}\\ |S|=r}}\;\prod_{j\in S}(1-p_j).
$$

Insert
$$
\frac{1}{N\binom{N-1}{r}}=\frac{r!(N-1-r)!}{N!}=\int_0^1 u^r(1-u)^{N-1-r}\,du
$$
and swap sum/integral (justified since the sums are finite), we have
$$
\mathbb{E}_\pi[X_\pi]=\int_0^1 \left[\sum_{r=0}^{N-1} e_r\,u^r(1-u)^{N-1-r}\right]du
$$
where $e_r=\sum_{|S|=r}\prod_{j\in S}(1-p_j)$ are elementary symmetric sums of $\{1-p_j\}_{j\neq k}$.

The generating function is
$$
\sum_{r=0}^{N-1} e_r t^r=\prod_{j\neq k}\bigl(1+(1-p_j)t\bigr)
$$
With $t=\frac{u}{1-u}$ and factoring $(1-u)^{N-1}$ we obtain
$$
\sum_{r=0}^{N-1} e_r\,u^r(1-u)^{N-1-r}=\prod_{j\neq k}(1-p_j u)
$$
so
$$
\mathbb{E}_\pi[X_\pi]=\int_0^1 \prod_{j\neq k}(1-p_j u)\,du
$$

There are $m$ terms with $p_j=a$ and $n$ with $p_j=b$, hence
$$
I_{m,n}(a,b)=\int_0^1(1-au)^m(1-bu)^n\,du
$$
Expanding $(1-au)^m(1-bu)^n$ and integrating termwise yields the stated double sum. The special cases follow by taking $n=0$ or $m=0$.

For a specific candidate $k$ with hit probability $p_k\in\{a,b\}$,
$$
\Pr(k\text{ is first hit})=p_k\cdot \mathbb{E}_\pi\!\left[\prod_{j\prec_\pi k}(1-p_j)\right]
=\begin{cases}
a\,I_{S_\ell-1,U_\ell}(a,b), & k\text{ a true sepset},\\
b\,I_{S_\ell,\,U_\ell-1}(a,b), & k\text{ a non-sepset}.
\end{cases}
$$
\end{proof}

\begin{lemma}[Level-$\ell$ identification probabilities]\label{accuracy:lemma4}
Condition on the partition $\{F_\ell\}$ from Lemma \ref{accuracy:lemma1}. Under Assumption \ref{assumption:SimpleCI},
$$
\Pr(E\mid D)=\sum_{\ell}\frac{\Pr(E\cap F_\ell)}{\Pr(D)}
=\sum_{\ell}\frac{\mathrm{PrevNoHit}_\ell}{\Pr(D)}\cdot \Pr(E \text{ via level }\ell \mid \text{ level } \ell \text{ has a hit}).
$$
For CPC (bucket exclusivity within a level) and PC (first hit within a level),
\begin{align*}
&\Pr_{\mathrm{CPC}}(\text{identified as collider}\mid D)\\
&\quad=\frac{1}{\Pr(D)}\sum_{\ell}\mathrm{PrevNoHit}_\ell
\underbrace{\beta^{S_{Z,\ell}}(1-\alpha)^{U_{Z,\ell}}}_{\text{no $Z$-hits at }\ell}
\underbrace{\big(1-\beta^{S_{\neg Z,\ell}}(1-\alpha)^{U_{\neg Z,\ell}}\big)}_{\text{some non-$Z$ hit}},\\
&\Pr_{\mathrm{CPC}}(Z\text{ in all saved sepsets}\mid D)\\
&\quad=\frac{1}{\Pr(D)}\sum_{\ell}\mathrm{PrevNoHit}_\ell
\beta^{S_{\neg Z,\ell}}(1-\alpha)^{U_{\neg Z,\ell}}
\big(1-\beta^{S_{Z,\ell}}(1-\alpha)^{U_{Z,\ell}}\big),\\
&\Pr_{\mathrm{PC}}(\text{identified as collider}\mid D)\\
&\quad=\frac{1}{\Pr(D)}\sum_{\ell}\mathrm{PrevNoHit}_\ell\Bigl[
S_{\neg Z,\ell}\,(1-\beta)\,I_{S_\ell-1,U_\ell}((1-\beta),\alpha)\\
&\qquad\qquad+U_{\neg Z,\ell}\,\alpha\,I_{S_\ell,U_\ell-1}((1-\beta),\alpha)\Bigr],\\
&\Pr_{\mathrm{PC}}(Z\text{ in saved sepsets}\mid D)\\
&\quad=\frac{1}{\Pr(D)}\sum_{\ell}\mathrm{PrevNoHit}_\ell\Bigl[
S_{Z,\ell}\,(1-\beta)\,I_{S_\ell-1,U_\ell}((1-\beta),\alpha)\\
&\qquad\qquad+U_{Z,\ell}\,\alpha\,I_{S_\ell,U_\ell-1}((1-\beta),\alpha)\Bigr].
\end{align*}
\end{lemma}

\begin{proof}
Condition on $F_\ell$ and apply CPC’s exclusivity or PC’s first-hit rule with the order-averaged factors from Lemma \ref{accuracy:lemma3}.
\end{proof}

\paragraph{Proofs for Theorem \ref{thrm:odds}} For CPC, under non-collider truth (sepsets in $\text{incZ}_\ell$), Lemma \ref{accuracy:lemma2} gives $S_{Z,\ell}=m$, $U_{Z,\ell}=0$, $S_{\neg Z,\ell}=0$, $U_{\neg Z,\ell}=n$. Plugging into Lemma \ref{accuracy:lemma4} yields $\Pr_{\mathrm{CPC}}(\text{collider at }\ell)=\mathrm{PrevNoHit}_\ell\,\beta^{m}\bigl[1-(1-\alpha)^n\bigr]$; under collider truth (sepsets in $\text{notZ}_\ell$), the symmetric expression is $\Pr_{\mathrm{CPC}}(Z \text{ in saved sepset at }\ell)=\mathrm{PrevNoHit}_\ell\,\beta^{n}\bigl[1-(1-\alpha)^m\bigr]$. Take the ratio to cancel $\mathrm{PrevNoHit}_\ell$. For CPC, with the same bucket counts and Lemma \ref{accuracy:lemma3}, $\Pr_{\mathrm{PC}}(\text{collider at }\ell)=\mathrm{PrevNoHit}_\ell\cdot n\cdot b\cdot I_{m,\,n-1}(a,b)$, and $\Pr_{\mathrm{PC}}(Z \text{ in saved sepset at }\ell)=\mathrm{PrevNoHit}_\ell\cdot m\cdot b\cdot I_{n,\,m-1}(a,b)$. Divide to cancel $\mathrm{PrevNoHit}_\ell$ and $b$.

For Level-$\ell$ ($\ell\ge 1$), the wrong-orientation odds under Assumptions \ref{assumption:SimpleCI},\ref{assumption:purity} can be obtained as
\begin{align*}
\mathcal R_\ell^{\mathrm{CPC}}
&=\beta^{\,m-n}\,\frac{1-(1-\alpha)^{n}}{1-(1-\alpha)^{m}}\\
\mathcal R_\ell^{\mathrm{PC}}
&=\frac{n}{m}\cdot
\frac{I_{m,\,n-1}(1-\beta,\alpha)}{I_{n,\,m-1}(1-\beta,\alpha)}
\end{align*}
where $I_{p,q}(a,b)=\int_0^1 (1-au)^p(1-bu)^q\,du$, $m=\text{incZ}_\ell$ and $n=\text{notZ}_\ell$.

Therefore, for early levels $\ell\ge1$ and $\ell\ll\frac{M}{2}$:
\begin{align*}
\text{CPC: }&
\mathcal R_\ell^{\mathrm{CPC}}\approx \beta^{m-n}\frac{n}{m}
=\beta^{\binom{M-1}{\ell-1}-\binom{M-1}{\ell}}\cdot\frac{M-\ell}{\ell}
\quad(\alpha\text{ small})\\
\text{PC: }&
\mathcal R_\ell^{\mathrm{PC}}\approx \frac{n(n+1)}{m(m+1)}
\left[1+\alpha\frac{(m-n)(m+n+1)}{(m+2)(n+2)}\right]\quad
(\alpha,\beta\text{ small}),
\end{align*}
Taking $\alpha, \beta = o\left(\frac{1}{M}\right)$ completes the proof for Theorem \ref{thrm:odds}. As $\binom{M-1}{\ell}$ grows with $\ell$ up to $M/2$, $n > m$, $\beta^{m-n} > 1$ as $\beta < 1$, so $\mathcal R_\ell^{\mathrm{CPC}} > 1$.
At zeroth order $\mathcal R_\ell^{\mathrm{PC}}\approx \dfrac{n(n+1)}{m(m+1)}
\approx \big(\dfrac{M-\ell}{\ell}\big)^2$ when $m,n$ are large, so $\mathcal R_\ell^{\mathrm{PC}}>1$ for $\ell<M/2$.

\begin{corollary}[small-$\alpha$ approximation for CPC]
For $\ell\ge1$ and small $\alpha$, as $1-(1-\alpha)^t=t\alpha+O(\alpha^2)$, we have 
$$
\mathcal R_\ell^{\mathrm{CPC}}=\beta^{\,m-n}\,\frac{1-(1-\alpha)^{n}}{1-(1-\alpha)^{m}}\approx \beta^{\,m-n}\,\frac{n}{m}=\beta^{\binom{M-1}{\ell-1}-\binom{M-1}{\ell}}\cdot\frac{M-\ell}{\ell},
$$
with relative error $O(\alpha)$.
\end{corollary}

\begin{corollary}[small-$\alpha$ and small-$\beta$ approximation for PC]
For $\ell\ge1$ and small $\alpha,\beta$,
$$
I_{m,n}(1-\beta,\alpha)
\;\approx\;
\frac{1}{m+1}\;+\;\frac{\beta}{m+1}\;-\;\frac{n\,\alpha}{(m+1)(m+2)}
\;+\;O(\alpha^2,\alpha\beta,\beta^2).
$$
Using the ratio expansion $(x+\delta_x)/(y+\delta_y)\approx (x/y)\,[1+(\delta_x/x)-(\delta_y/y)]$, we have
$$
\mathcal R_\ell^{\mathrm{PC}}
\;\approx\;
\frac{n(n+1)}{m(m+1)}\left[
1\;+\;\alpha\;\frac{(m-n)(m+n+1)}{(m+2)(n+2)}
\right]
\;+\;O(\alpha^2,\alpha\beta,\beta^2).
$$
Zeroth order (ignore $\alpha,\beta$) is 
$$
\mathcal R_\ell^{\mathrm{PC}}\approx \frac{n(n+1)}{m(m+1)}
= \frac{M-\ell}{\ell}\cdot \frac{n+1}{m+1}.
$$
and $\frac{n+1}{m+1}$ has no tidy closed form in $M,\ell$. For large $m,n$, 
$\frac{n+1}{m+1}\approx\frac{n}{m}$, giving $\ \mathcal R_\ell^{\mathrm{PC}}\approx \big(\frac{M-\ell}{\ell}\big)^2$.
\end{corollary}

\begin{proof}
$$
I_{m,n}(1-\beta,\alpha)=\int_0^1\!\bigl(1-(1-\beta)u\bigr)^m\bigl(1-\alpha u\bigr)^n\,du.
$$
Write
\begin{align*}
\bigl(1-(1-\beta)u\bigr)^m &= (1-u)^m\Bigl(1+\frac{\beta u}{1-u}\Bigr)^m \approx (1-u)^m\Bigl(1+m\frac{\beta u}{1-u}\Bigr),\\
(1-\alpha u)^n&\approx 1-n\alpha u,
\end{align*}
keeping terms up to $O(\alpha,\beta)$ and dropping $O(\alpha^2,\alpha\beta,\beta^2)$ then multiplying (and ignoring the $\alpha\beta$ cross-term):
$$
\bigl(1-(1-\beta)u\bigr)^m(1-\alpha u)^n \approx (1-u)^m + m\beta\,u(1-u)^{m-1} - n\alpha\,u(1-u)^m.
$$
Integrate termwise, we have
$$
I_{m,n}(1-\beta,\alpha)\;\approx\;\frac{1}{m+1}\;+\;\frac{\beta}{m+1}\;-\;\frac{n\,\alpha}{(m+1)(m+2)}
\;+\;O(\alpha^2,\alpha\beta,\beta^2),
$$
Since
$$
\mathcal R_\ell^{\mathrm{PC}}=\frac{n}{m}\cdot\frac{I_{m,\,n-1}(1-\beta,\alpha)}{I_{n,\,m-1}(1-\beta,\alpha)}, \quad m=\text{incZ}_\ell,\; n=\text{notZ}_\ell.
$$

Apply the expansion:
\begin{align*}
I_{m,\,n-1}(1-\beta,\alpha)&\approx \frac{1}{m+1}+\frac{\beta}{m+1}-\frac{(n-1)\alpha}{(m+1)(m+2)},\\
I_{n,\,m-1}(1-\beta,\alpha)&\approx \frac{1}{n+1}+\frac{\beta}{n+1}-\frac{(m-1)\alpha}{(n+1)(n+2)}.
\end{align*}

Use the first-order ratio expansion $\frac{x+\delta_x}{y+\delta_y}\approx \frac{x}{y}\Bigl[1+\frac{\delta_x}{x}-\frac{\delta_y}{y}\Bigr]$, where $x=\frac{1}{m+1}$, $y=\frac{1}{n+1}$. As $\frac{\delta_x}{x}=\beta-\frac{(n-1)\alpha}{m+2}$, $\frac{\delta_y}{y}=\beta-\frac{(m-1)\alpha}{n+2}$, the ratio is approximated by $\frac{n+1}{m+1}\left[1+\alpha\left(\frac{m-1}{n+2}-\frac{n-1}{m+2}\right)\right]$.

Therefore
$$
\mathcal R_\ell^{\mathrm{PC}}\approx\frac{n(n+1)}{m(m+1)}\left[1+\alpha\,\frac{(m-n)(m+n+1)}{(m+2)(n+2)}\right]+O(\alpha^2,\alpha\beta,\beta^2).
$$

At zeroth order (ignore $\alpha,\beta$), 
$$
\mathcal R_\ell^{\mathrm{PC}}\approx \frac{n(n+1)}{m(m+1)}
=\Bigl(\frac{n}{m}\Bigr)\Bigl(\frac{n+1}{m+1}\Bigr).
$$
Since $n/m=\binom{M-1}{\ell}/\binom{M-1}{\ell-1}=\frac{M-\ell}{\ell}$,
$$
\mathcal R_\ell^{\mathrm{PC}}
\approx \frac{M-\ell}{\ell}\cdot\frac{n+1}{m+1},
$$

and $\frac{n+1}{m+1}$ has no tidy closed form in $M,\ell$. For large $m,n$, 
$\frac{n+1}{m+1}\approx\frac{n}{m}$, giving $\ \mathcal R_\ell^{\mathrm{PC}}\approx \big(\frac{M-\ell}{\ell}\big)^2$.
\end{proof}

\newpage
\section{Numerical experiments on non-collider / collider identification FPRs\label{subsec:numerical}}
We set max searching layer as $l = 3$, $\alpha = 0.05$ and $\beta = 0.1$, plugging in the number of nodes, arcs and average degrees of datasets, we can accordingly compute the expected FPRs as Table \ref{tab:expected-fprs}.
\begin{table}[htb!]
\centering
\caption{Expected false positive rates (FPRs) during identifications.}
\label{tab:expected-fprs}
\begin{tabular}{lrrrr}
\toprule
\textbf{network} & \textbf{PC\,(colliders-first)} & \textbf{PC\,(nonc-first)} & \textbf{CPC\,(colliders-first)} & \textbf{CPC\,(nonc-first)} \\
\midrule
asia       & 0.177849 & 0.000926 & 0.071491 & $5.\times10^{-8}$ \\
alarm      & 0.583846 & 0.000159 & 0.128392 & $5.\times10^{-37}$ \\
cancer     & 0.102895 & 0.001906 & 0.059310 & $6.\times10^{-5}$ \\
child      & 0.399529 & 0.000309 & 0.105279 & $5.\times10^{-20}$ \\
hailfinder & 0.700605 & 0.000103 & 0.138733 & $5.\times10^{-56}$ \\
hepar2     & 0.755133 & 0.000082 & 0.141944 & $5.\times10^{-70}$ \\
insurance  & 0.488650 & 0.000222 & 0.117261 & $5.\times10^{-27}$ \\
mildew     & 0.567216 & 0.000168 & 0.126597 & $5.\times10^{-35}$ \\
water      & 0.540173 & 0.000185 & 0.123536 & $5.\times10^{-32}$ \\
win95pts   & 0.773360 & 0.000075 & 0.142753 & $5.\times10^{-76}$ \\
\bottomrule
\end{tabular}
\end{table}

\newpage
\section{Prompting Templates and Parsing Rule}
\label{app:prompt}

\subsection{Answer Tag Parsing}
We extract the final choice using the following case-insensitive regular expression, which returns a single capital letter in \{A,B,C,D,E\}:
\begin{lstlisting}[language=Python,caption={Regex for parsing the <Answer> tag.}]
_ans_re = re.compile(r"<\s*answer\s*>\s*([ABCDE])\s*<\s*/\s*answer\s*>", re.I)
\end{lstlisting}

\subsection{Prompt templates}
Placeholders in braces are programmatically substituted (e.g., \verb|{u}|, \verb|{v}|, \verb|{data_desc}|).
\begin{lstlisting}[caption={Full chain-of-thought selection template.}]
You are a senior researcher in causal discovery. We are studying the following dataset:

{data_desc}

The two target variables under review are {u} and {v}.

Conditional-independence tests mentioning these variables:

{ci_bullets}

Neighbour chain(s) that must normally remain non-collider:

{chains}

The nodes involved are described as below: 

{node_desc}

Choose one explanation that best fits domain knowledge and/or decides a CI test is unreliable (avoid selecting D or E unless other options are strongly against common sense):

A. Undecided. We don't know enough to confidently pick a directionality.
B. Changing the state of {u} causally affects {v}, and {v} causally affects {u_theOther_2v}.
C. Changing the state of {v} causally affects {u}, and {u} causally affects {v_theOther_2u}.
D. Changing the state of {u} causally affects {v}, and {u_theOther_2v} also causally affects {v}, **violating corresponding CI tests**.
E. Changing the state of {v} causally affects {u}, and {v_theOther_2u} also causally affects {u}, **violating corresponding CI tests**.

Think step-by-step before selecting:
1. Mechanisms - What known causal pathways (biological, physical, etc.) support each direction?
2. Counterfactual test - What would happen if we intervened on one node? What would we expect?
3. Empirical check - Point to one key piece of information that favors/weakens a direction.
4. Comparison - Briefly weigh A vs B vs C vs D vs E and choose the most plausible.

Return exactly three lines:
1. Reasoning in support of one direction.
2. Reasoning against the weaker/less plausible direction.
3. Final choice:  <Answer>A/B/C/D/E</Answer> 
\end{lstlisting}

\subsection{Template When Only \texorpdfstring{$v \to u$}{v->u} Ancillary Edge Is Possible (\texttt{\_CHAIN\_PROMPT\_TMPL\_None2u})}
\begin{lstlisting}[caption={Restricted template (None2u).}]
You are a senior researcher in causal discovery. We are studying the following dataset:

{data_desc}

The two target variables under review are {u} and {v}.

Conditional-independence tests mentioning these variables:

{ci_bullets}

Neighbour chain(s) that must normally remain non-collider:

{chains}

The nodes involved are described as below: 

{node_desc}

Choose one explanation that best fits domain knowledge and/or decides a CI test is unreliable (avoid selecting D unless other options are strongly against common sense):

A. Undecided. We don't know enough to confidently pick a directionality.
B. Changing the state of {u} causally affects {v}, and {v} causally affects {u_theOther_2v}.
C. Changing the state of {v} causally affects {u}.
D. Changing the state of {u} causally affects {v}, and {u_theOther_2v} also causally affects {v}, **violating corresponding CI tests**.

Think step-by-step before selecting:
1. Mechanisms - What known causal pathways (biological, physical, etc.) support each direction?
2. Counterfactual test - What would happen if we intervened on one node? What would we expect?
3. Empirical check - Point to one key piece of information that favors/weakens a direction.
4. Comparison - Briefly weigh A vs B vs C vs D and choose the most plausible.

Return exactly three lines:
1. Reasoning in support of one direction.
2. Reasoning against the weaker/less plausible direction.
3. Final choice:  <Answer>A/B/C/D</Answer> 
\end{lstlisting}

\subsection{Template When Only \texorpdfstring{$u \to v$}{u->v} Ancillary Edge Is Possible (\texttt{\_CHAIN\_PROMPT\_TMPL\_None2v})}
\begin{lstlisting}[caption={Restricted template (None2v).}]
You are a senior researcher in causal discovery. We are studying the following dataset:

{data_desc}

The two target variables under review are {u} and {v}.

Conditional-independence tests mentioning these variables:

{ci_bullets}

Neighbour chain(s) that must normally remain non-collider:

{chains}

The nodes involved are described as below: 

{node_desc}

Choose one explanation that best fits domain knowledge and/or decides a CI test is unreliable (avoid selecting D unless other options are strongly against common sense):

A. Undecided. We don't know enough to confidently pick a directionality.
B. Changing the state of {u} causally affects {v}.
C. Changing the state of {v} causally affects {u}, and {u} causally affects {v_theOther_2u}.
D. Changing the state of {v} causally affects {u}, and {v_theOther_2u} also causally affects {u}, **violating corresponding CI tests**.

Think step-by-step before selecting:
1. Mechanisms - What known causal pathways (biological, physical, etc.) support each direction?
2. Counterfactual test - What would happen if we intervened on one node? What would we expect?
3. Empirical check - Point to one key piece of information that favors/weakens a direction.
4. Comparison - Briefly weigh A vs B vs C vs D and choose the most plausible.

Return exactly three lines:
1. Reasoning in support of one direction.
2. Reasoning against the weaker/less plausible direction.
3. Final choice:  <Answer>A/B/C/D</Answer> 
\end{lstlisting}

\subsection{Template Without CI/Neighbour Context (\texttt{\_CHAIN\_PROMPT\_TMPL\_None})}
\begin{lstlisting}[caption={Minimal template (None).}]
You are a senior researcher in causal discovery. We are studying the following dataset:

{data_desc}

The two target variables under review are {u} and {v}.

The nodes involved are described as below: 

{node_desc}

Choose one explanation that best fits domain knowledge:

A. Undecided. We don't know enough to confidently pick a directionality.
B. Changing the state of {u} causally affects {v}.
C. Changing the state of {v} causally affects {u}.

Think step-by-step before selecting:
1. Mechanisms - What known causal pathways (biological, physical, etc.) support each direction?
2. Counterfactual test - What would happen if we intervened on one node? What would we expect?
3. Empirical check - Point to one key piece of information that favors/weakens a direction.
4. Comparison - Briefly weigh A vs B vs C and choose the most plausible.

Return exactly three lines:
1. Reasoning in support of one direction.
2. Reasoning against the weaker/less plausible direction.
3. Final choice:  <Answer>A/B/C</Answer> 
\end{lstlisting}

\newpage
\section{Reasoning examples}

Below are 2 examples of the LLM reasoning trace in support of a direction. The question asked is from Appendix \ref{app:prompt}2. 

\begin{lstlisting}[caption={Reasoning with large p-values}]
Reasoning in support of one direction.  
B is supported by domain knowledge: 'either' (tuberculosis or lung cancer) can cause dyspnoea as a symptom. Intervening on 'either' (e.g., treating the disease) would alter dyspnoea, aligning with a causal direction from 'either' to dyspnoea. The neighbor chain being non-collider further supports this directionality.  

Reasoning against the weaker/less plausible direction.  
C and D contradict domain knowledge: dyspnoea is an effect, not a cause, of diseases like tuberculosis or lung cancer. The CI tests (p-values ~0.7-1) indicate no direct link between dyspnoea and the other variables given 'either', invalidating claims of dyspnoea influencing 'either' or its dependencies.  

Final choice: <Answer>B</Answer>

\end{lstlisting}

\begin{lstlisting}[caption={Reasoning with small p-values}]
1. Reasoning in support of one direction.  
Stroke Volume (SV) directly determines Cardiac Output (CO) via the relationship CO = SV x Heart Rate. CO, in turn, is a primary determinant of Blood Pressure (BP), as BP ~= CO x Peripheral Resistance. The CI tests confirm that once CO is known, SV and BP become independent (removing the Stroke Volume-Blood Pressure edge), and Hypovolemia/Left-Ventricular Failure influence CO only via SV (removing their direct links to CO). This aligns with the physiological mechanism where SV -> CO -> BP.  

2. Reasoning against the weaker/less plausible direction.  
Option C reverses the causal relationship between CO and SV, which contradicts the physiological definition of CO as a product of SV and Heart Rate. Additionally, Stroke Volume influencing diagnostic variables like Hypovolemia is implausible, as these conditions are root causes (diagnostic variables have no predecessors). Options D and E imply violations of CI tests, which are explicitly rejected by the p-values, leaving B as the only direction consistent with both domain knowledge and statistical evidence.  

3. Final choice: <Answer>B</Answer>

\end{lstlisting}

\newpage
\section{Assumptions of related work}
\label{app:assumptions}

\textbf{Vashishtha} \citep{vashishtha2025causal} treat the LLM  as an imperfect expert queried for causal direction, and assume causal order (a topological order) is a more stable interface for expert knowledge (“we propose that causal order be used as a more stable output interface”), as opposed to directly prompting undirected edges for their orientations. Pairwise LLM querying produces inconsistencies (“pairwise prompts lead to multiple cycles”), so they add an explicit local consistency constraint by prompting on triplets and “instructing the LLM to avoid cycles within this triplet,” then vote over redundant triplets; they also analyze an “imperfect expert with an error $\epsilon$ on each prediction” and assumes "the $\epsilon$ expert to have error probability exactly equal to $\epsilon$". 

\textbf{Causal Disco} \citep{long2307causal} explicitly models the LLM as an imperfect expert oracle ($E(\mu_i,\mu_j$)) that consumes metadata (names/descriptions) and returns a hypothesized orientation for an edge. A core assumption of their setup is that the true Markov equivalence class is already available (“We assume that ($M_{G^\star}$) is known, e.g., that it has been obtained via some causal discovery algorithm”), and the LLM is used only to orient edges that are ambiguous within that MEC. They further assume a simple noise model for LLM mistakes: the expert’s response depends only on the true orientation and is “incorrect with constant probability ($\epsilon$),” which enables Bayesian reasoning over consistent orientations while controlling risk (they target keeping ($G^\star$) in the remaining set with probability at least ($1-\eta$)). 

\textbf{LLM-CD} \citep{ban2025integrating} is another imperfect expert method. Concretely, they assume LLM reasoning should be restricted (“an accuracy-oriented prompting strategy restricts causal analysis to a reliable range”), because “LLM-based causal reasoning on all pairs…[is]” unreliable when mechanisms are unclear. Instead of edge constraints they translate LLM outputs into ancestral/path constraints: “We do not directly specify the existence of edges… [LLM] knowledge [is] qualitative and can represent indirect interactions.” A key consideration of this approach is the reliance on a manually tuned confidence threshold for determining LLM accuracy, which as shown in their paper may vary between datasets.

\textbf{Jiralerspong} \citep{jiralerspong2024efficientcausalgraphdiscovery} assumes an LLM can answer causal queries from metadata alone and explicitly “utilize LLMs to respond to causal queries solely based on metadata… without accessing the numerical observations,” analogizing this to how human experts build graphs from descriptions. Their main additional assumption for scalability is that full-graph discovery can be organized as a traversal: they construct the graph via BFS with only (O(n)) prompts by asking (i) root-node style questions (“Ask which variables are not caused by any other variables”) and (ii) expansion questions (“ask which variables it causes”), inserting edges unless they “form a cycle” to enforce the DAG constraint. 

\textbf{ILS-CSL} \citep{ban2023causal} assumes LLM judgments are most reliable when scoped to checking candidate edges suggested by data. Prior LLM-for-CSL approaches can yield “unreliable constraints from imperfect LLM inferences” and that full pairwise prompting of all nodes in the graph is costly; their key design choice is to “focus LLMs on verifying direct causal relationships already suggested by the data,” i.e., use LLMs as an iterative supervisor of a data-driven CSL loop. They assume convergence can be defined operationally: “The iteration concludes when the LLM-based inferences and data-driven CSL align".

\textbf{SCP} \citep{takayama2024integrating} assumes LLMs can contribute background causal knowledge without needing data, and that this knowledge becomes more useful when conditioned on statistical signals from an initial causal discovery run. The LLM is “prompted with the results of the SCD” and “evaluates the probability of the causal relationships considering both the domain knowledge and the statistical characteristics suggested by SCD.” They assume pairwise causal answers can be operationalized as a prior: the LLM “judges whether causal relationships exist between all pairs… with ‘yes’ or ‘no’,” and “probabilities of the responses… [are] transformed into the prior knowledge matrix… reapplied to SCD.”


\newpage
\input{checklist.tex}

\end{document}

%% file: checklist.tex
\section*{NeurIPS Paper Checklist}

The checklist is designed to encourage best practices for responsible machine learning research, addressing issues of reproducibility, transparency, research ethics, and societal impact. Do not remove the checklist: {\bf The papers not including the checklist will be desk rejected.} The checklist should follow the references and follow the (optional) supplemental material.  The checklist does NOT count towards the page
limit. 

Please read the checklist guidelines carefully for information on how to answer these questions. For each question in the checklist:
\begin{itemize}
    \item You should answer \answerYes{}, \answerNo{}, or \answerNA{}.
    \item \answerNA{} means either that the question is Not Applicable for that particular paper or the relevant information is Not Available.
    \item Please provide a short (1--2 sentence) justification right after your answer (even for \answerNA). 
\end{itemize}

{\bf The checklist answers are an integral part of your paper submission.} They are visible to the reviewers, area chairs, senior area chairs, and ethics reviewers. You will also be asked to include it (after eventual revisions) with the final version of your paper, and its final version will be published with the paper.

The reviewers of your paper will be asked to use the checklist as one of the factors in their evaluation. While \answerYes{} is generally preferable to \answerNo{}, it is perfectly acceptable to answer \answerNo{} provided a proper justification is given (e.g., error bars are not reported because it would be too computationally expensive'' or ``we were unable to find the license for the dataset we used''). In general, answering \answerNo{} or \answerNA{} is not grounds for rejection. While the questions are phrased in a binary way, we acknowledge that the true answer is often more nuanced, so please just use your best judgment and write a justification to elaborate. All supporting evidence can appear either in the main paper or the supplemental material, provided in appendix. If you answer \answerYes{} to a question, in the justification please point to the section(s) where related material for the question can be found.

IMPORTANT, please:
\begin{itemize}
    \item {\bf Delete this instruction block, but keep the section heading ``NeurIPS Paper Checklist"},
    \item  {\bf Keep the checklist subsection headings, questions/answers and guidelines below.}
    \item {\bf Do not modify the questions and only use the provided macros for your answers}.
\end{itemize}


\begin{enumerate}

\item {\bf Claims}
    \item[] Question: Do the main claims made in the abstract and introduction accurately reflect the paper's contributions and scope?
    \item[] Answer: \answerYes{} 
    \item[] Justification: Each claim is supported by an appropriate results section.

\item {\bf Limitations}
    \item[] Question: Does the paper discuss the limitations of the work performed by the authors?
    \item[] Answer: \answerYes{} 
    \item[] Justification: We have an explicit limitations section.

\item {\bf Theory assumptions and proofs}
    \item[] Question: For each theoretical result, does the paper provide the full set of assumptions and a complete (and correct) proof?
    \item[] Answer: \answerYes{} 
    \item[] Justification: Yes, this is a core contribution we make clear throughout the main text and appendix.

    \item {\bf Experimental result reproducibility}
    \item[] Question: Does the paper fully disclose all the information needed to reproduce the main experimental results of the paper to the extent that it affects the main claims and/or conclusions of the paper (regardless of whether the code and data are provided or not)?
    \item[] Answer: \answerYes{} 
    \item[] Justification: We provide dataset details and documentation as well as our code for running our method.

\item {\bf Open access to data and code}
    \item[] Question: Does the paper provide open access to the data and code, with sufficient instructions to faithfully reproduce the main experimental results, as described in supplemental material?
    \item[] Answer: \answerYes{} 
    \item[] Justification: We upload our code and provide data documentation.

\item {\bf Experimental setting/details}
    \item[] Question: Does the paper specify all the training and test details (e.g., data splits, hyperparameters, how they were chosen, type of optimizer) necessary to understand the results?
    \item[] Answer: \answerNA{} 
    \item[] Justification: We do not have a train/test split setup here.

\item {\bf Experiment statistical significance}
    \item[] Question: Does the paper report error bars suitably and correctly defined or other appropriate information about the statistical significance of the experiments?
    \item[] Answer: \answerYes{} 
    \item[] Justification: We report confidence intervals and results average across replicates.

\item {\bf Experiments compute resources}
    \item[] Question: For each experiment, does the paper provide sufficient information on the computer resources (type of compute workers, memory, time of execution) needed to reproduce the experiments?
    \item[] Answer: \answerYes{} 
    \item[] Justification: We have an explicit section talking about cost and runtime.
    
\item {\bf Code of ethics}
    \item[] Question: Does the research conducted in the paper conform, in every respect, with the NeurIPS Code of Ethics \url{https://neurips.cc/public/EthicsGuidelines}?
    \item[] Answer: \answerYes{} 
    \item[] Justification: We do not violate any ethical policies.

\item {\bf Broader impacts}
    \item[] Question: Does the paper discuss both potential positive societal impacts and negative societal impacts of the work performed?
    \item[] Answer: \answerYes{} 
    \item[] Justification: We motivate the uses of causal discovery.
    
\item {\bf Safeguards}
    \item[] Question: Does the paper describe safeguards that have been put in place for responsible release of data or models that have a high risk for misuse (e.g., pre-trained language models, image generators, or scraped datasets)?
    \item[] Answer: \answerNA{} 
    \item[] Justification: We do not believe there is any potential for misuse.

\item {\bf Licenses for existing assets}
    \item[] Question: Are the creators or original owners of assets (e.g., code, data, models), used in the paper, properly credited and are the license and terms of use explicitly mentioned and properly respected?
    \item[] Answer: \answerYes{} 
    \item[] Justification: All data and baselines are sufficiently cited.

\item {\bf New assets}
    \item[] Question: Are new assets introduced in the paper well documented and is the documentation provided alongside the assets?
    \item[] Answer: \answerNA{} 
    \item[] Justification: We do not introduce any new assets.

\item {\bf Crowdsourcing and research with human subjects}
    \item[] Question: For crowdsourcing experiments and research with human subjects, does the paper include the full text of instructions given to participants and screenshots, if applicable, as well as details about compensation (if any)? 
    \item[] Answer: \answerNA{} 
    \item[] Justification: We do not do crowdsourcing.

\item {\bf Institutional review board (IRB) approvals or equivalent for research with human subjects}
    \item[] Question: Does the paper describe potential risks incurred by study participants, whether such risks were disclosed to the subjects, and whether Institutional Review Board (IRB) approvals (or an equivalent approval/review based on the requirements of your country or institution) were obtained?
    \item[] Answer: \answerNA{} 
    \item[] Justification: We do not need IRB approval.

\item {\bf Declaration of LLM usage}
    \item[] Question: Does the paper describe the usage of LLMs if it is an important, original, or non-standard component of the core methods in this research? Note that if the LLM is used only for writing, editing, or formatting purposes and does \emph{not} impact the core methodology, scientific rigor, or originality of the research, declaration is not required.
    \item[] Answer: \answerYes{} 
    \item[] Justification: Yes, it is a core contribution of our paper and well documented throughout the main text, code, and appendix.

\end{enumerate}